\documentclass{article}

\usepackage[nonatbib, preprint]{neurips_2021}

\usepackage[utf8]{inputenc} 
\usepackage[T1]{fontenc}    
\usepackage{hyperref}       
\usepackage{url}            
\usepackage{booktabs}       
\usepackage{amsfonts,amsmath,bm,amsthm,bbm,multirow}       
\usepackage{nicefrac}       
\usepackage{microtype}      
\usepackage{xcolor}         
\usepackage{graphicx}
\usepackage{float}
\usepackage{algorithm}
\usepackage{algorithmic}
\usepackage{subfig}

\newtheorem{theorem}{Theorem}
\newtheorem{lemma}{Lemma}
\newtheorem{definition}{Definition}

\def\Ry{R\'enyi }
\def\ie{{\em i.e.}}

\title{Intrinsically-Motivated Reinforcement Learning: A Brief Introduction}

\author{%
  Mingqi Yuan \\
  Tencent Robotics X Lab\\
  Shenzhen, CN 518172 \\
  \texttt{mingqiyuan@tencent.com} \\
}

\begin{document}

\maketitle

\begin{abstract}
  Reinforcement learning (RL) is one of the three basic paradigms of machine learning. It has demonstrated impressive performance in many complex tasks like Go and StarCraft, which is increasingly involved in smart manufacturing and autonomous driving. However, RL consistently suffers from the exploration-exploitation dilemma. In this paper, we investigated the problem of improving exploration in RL and introduced the intrinsically-motivated RL. In sharp contrast to the classic exploration strategies, intrinsically-motivated RL utilizes the intrinsic learning motivation to provide sustainable exploration incentives. We carefully classified the existing intrinsic reward methods and analyzed their practical drawbacks. Moreover, we proposed a new intrinsic reward method via R\'enyi state entropy maximization, which overcomes the drawbacks of the preceding methods and provides powerful exploration incentives. Finally, extensive simulation demonstrated that the proposed module achieve superior performance with higher efficiency and robustness.
\end{abstract}

\newpage
\tableofcontents
\newpage

\section{Introduction}
\label{chap:intro}

\subsection{Background}
\label{sec:intro-background}     

Learning behavior permeates the whole life of human beings. However, it is extremely intricate to characterize or summarize the nature of learning. A natural idea is that we learn by interacting with our environment \cite{sutton2018reinforcement}. For instance, we observes the road condition in real time to adjust direction and speed to drive safely and smoothly. Take conversation as another example, we respond according to the statement of the other party. In this two cases, we are acutely aware of how our environment reacts to what we do, and try to influence what happens through our actions. Learning from interaction is a fundamental idea underlying almost all theories of learning and intelligence.

Reinforcement learning (RL) is a computational approach to learning from interaction. This term comes from behavioral psychology, which indicates that organisms take favorable strategies more frequently to draw on the advantages and avoid disadvantages. For instance, a general will make various deployments following some strategy in a battle. If one of his decisions leads him to victory, he will use this strategy more in future battles. \cite{pavlov1928lectures} first introduced the term "reinforcement" to describe the phenomenon that specific incentives make organisms tend to adopt certain strategies. Furthermore, \cite{michael1975positive} classified reinforcements as positive reinforcements and negative reinforcements. Positive reinforcements make organisms tend to gain more benefits, while negative reinforcements make organisms avoid damage. RL has become one of the three paradigms of machine learning (ML), and achieved great success like AlphaGo \cite{silver2018general}.

A typical RL scenario is composed of an agent and an environment. The interactions between the agent and the environment are often modeled as a Markov decision process (MDP), which comprises a state space, an action space, a reward function and a transition probability \cite{norris1998markov}. In each time step, the agent observes the state of environment before selecting an action from the action space. After that, the state-action pair will be evaluated by the reward function. Finally, the objective of RL is to learn the optimal policy that maximizes the accumulative rewards in the long run. For policy learning, RL developed a large number of approaches that can be broadly categorized into optimal value algorithms and policy gradient algorithms. The former learns the optimal policy via estimating the optimal value functions, such as Q-learning \cite{watkins1992q}, while the latter learns the optimal policy via policy gradient theorem, such as trust-region-policy-optimization (TRPO) and proximal-policy-optimization (PPO) \cite{schulman2017proximal,schulman2015trust}. Take Q-learning for instance, its convergence condition is to visit all possible state-action pairs infinitely. However, many existing RL algorithms suffer from insufficient exploration mechanism. The agent cannot keep exploring the environment across episodes, especially when handling the complex environment with high-dimensional observations. As a result, the learned policy will prematurely fall into local optima after finite iterations \cite{stadie2015incentivizing}. 

The aforementioned phenomenon is called exploration-exploitation dilemma in RL. A representative example for demonstrating this dilemma is the $K$-armed bandit problem \cite{auer2002finite}. Each slot machine has a certain payout probability unknown to the gambler. The goal for the gambler is to identify the slot machine of the highest payout probability with the minimum number of attempts. One trivial approach for the gambler is to follow the exploration-only method by uniformly attempting all slot machines. However, such an approach usually requires a large number of attempts before the best slot machine is identified. On the other hand, the gambler can follow the exploitation method by focusing on the slot machine that gave the highest payout after a limited number of attempts. Clearly, this so-called the “best” slot machine may not be optimal at all. Both these methods cannot effectively maximize the payout with limited attempts. Thus, a comprehensive algorithm must be found by striking an expedient balance between exploitation and exploration.

Inspired by the discussions above, this thesis focuses on the exploration problem of RL to establish efficient and robust exploration methods. More specifically, we delve into the intrinsically-motivated RL that utilizes intrinsic learning motivation to improve exploration. In the following contents, we will introduce the identity of intrinsic learning motivation and transform the concept into a computational model. Finally, the effectiveness of the proposed methods is verified using realistic control tasks to demonstrate their advantages.

\subsection{Related Work}
\label{sec:intro-related}
The exploration problem of RL has been extensively studied. Exploration methods encourage RL agents to fully explore the state space in various manners, for instance by rewarding surprise \cite{schmidhuber2006developmental,schmidhuber2010formal}, curiosity \cite{pathak2017curiosity, burda2018exploration}, information gain (IG) \cite{little2013learning, houthooft2016vime}, feature control \cite{dilokthanakul2019feature} or diversity \cite{eysenbach2018diversity}. Another class of exploration methods following the insight of Thompson sampling \cite{ostrovski2017count,o2018uncertainty,osband2016deep,plappert2017parameter}. For instance, \cite{osband2016deep} utilizes a family of randomized Q-functions trained on bootstrapped data to choose actions, while \cite{plappert2017parameter} inserts noise into parameter space to encourage exploration. In this thesis, we focus on intrinsic motivation methods, which are extensively used and proven to be effective for extensive hard-exploration tasks. Intrinsic motivation is very useful in leading the exploration of RL agents, especially in environments where the extrinsic rewards are sparse or missing \cite{oudeyer2009intrinsic}. In particular, most effective and popular intrinsic motivation can be broadly classified into two approaches, namely state novelty-based methods and prediction error-based methods.

State novelty-based methods encourage the agent to visit as many novel states as possible. A representative approach is the count-based exploration that utilizes state visitation counts as exploration bonus in tabular settings \cite{strehl2008analysis}. \cite{tang2017exploration,martin2017count} further extend such methods to complex environments with high-dimensional observations. \cite{bellemare2016unifying, ostrovski2017count} designs a state pseudo-count method based on context-tree switching density model, while \cite{ostrovski2017count} uses neural network as a state density estimator. In contrast, prediction error-based methods encourage the agent to fully explore the environment to reduce the uncertainty or error in predicting the consequences of its own actions. For instance, \cite{burda2018exploration} utilizes the prediction error of random networks as exploration bonus to reward infrequently-seen states. \cite{stadie2015incentivizing} designs an intrinsic reward as the prediction error of predicting the encoded next state based on the current state-action pair. \cite{pathak2017curiosity} further introduces an embedding network to learn the representation of state space and designs a inverse-forward pattern to generate intrinsic rewards. However, the prediction-error based methods suffer from the television dilemma, \ie, the agent may hover in a novel state and stop exploration \cite{savinov2018episodic}. Moreover, the two kinds of methods generate vanishing rewards and cannot provide sustainable exploration incentives. In the following sections, we will dive into these problems and propose efficient solutions.

\subsection{Contributions}
\label{sec:intro-contri}
The main contributions of this thesis are summarized as follows:
\begin{itemize}
	\item We systematically introduce and analyze the exploration-exploitation dilemma in RL, and summarize the traditional strategic exploration algorithms;
	\item We detail the intrinsic learning motivation, and categorize various intrinsic reward generation methods. We further analyze the practical drawbacks of the existing intrinsic reward methods, and introduce the corresponding solutions;
	\item We designed a brand new intrinsic reward method entitled \textbf{R}\'eny\textbf{I} \textbf{S}tate \textbf{E}ntropy (RISE), which overcomes the drawbacks of the preceding methods and provides powerful exploration incentives;
	\item Finally, extensive simulation is performed to compare	the performance of RISE against existing methods using both discrete and continuous control tasks as well as several hard exploration games. Simulation results confirm that the proposed module achieve superior performance with high efficiency and robustness.
\end{itemize}

\subsection{Organization}
\label{sec:intro-orga}

The remainder of the thesis is organized as follows:
\begin{itemize}
	\item Chapter \ref{chap:one} introduces the fundamentals of RL, including Markov decision process and some representative RL algorithms. These definitions and algorithms will run through the full text;
	\item Chapter \ref{chap:two} first demonstrates the exploration-exploitation dilemma in RL. After that, we introduce several classical exploration strategies both for tabular setting and deep RL, such as $\epsilon$-greedy policy and Boltzmann exploration.
	\item Chapter \ref{chap:three} first analyze the intrinsic learning motivation of the agent before introducing the reward shaping method in RL. After that, we discussed two important kinds of intrinsic reward methods, namely state novelty-based and prediction error-based approaches. Finally, we investigated some practical drawbacks of the discussed methods and proposed the corresponding solutions.
	\item Chapter \ref{chap:four} first analyzes the exploration problem from a new perspective before proposing a powerful intrinsic reward method, which provides high-quality intrinsic rewards via state entropy maximization. Extensive simulation demonstrates that this method can provide sustainable exploration incentives with less computational complexity and higher robustness.
	\item Chapter \ref{chap:con} concludes the full thesis and discusses the future work.
\end{itemize}

\section{Fundamentals of Reinforcement Learning}
\label{chap:one}

\begin{figure}[h]
	\centering
	\includegraphics[]{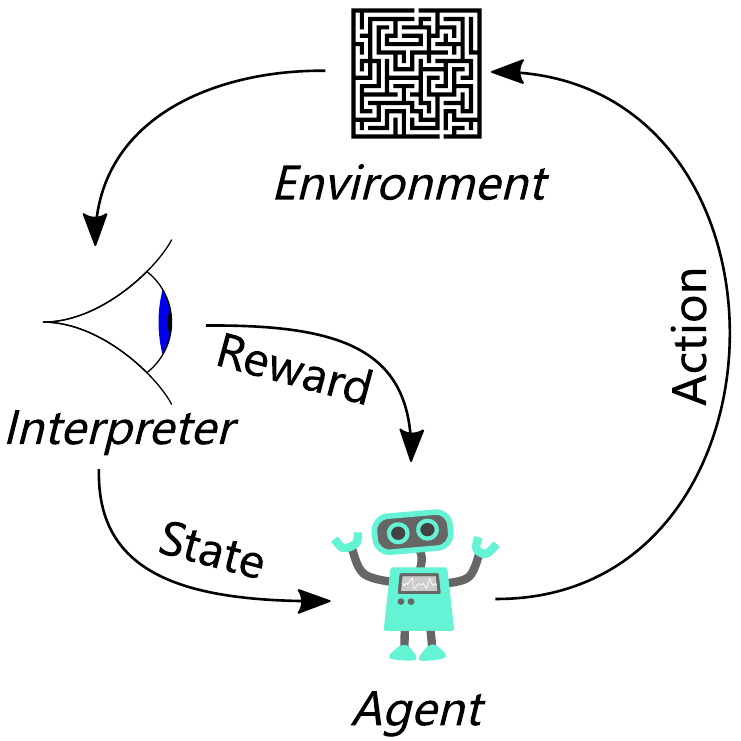}
	\caption{The typical framing of a reinforcement learning scenario.}
	\label{fig:rl}
\end{figure}

\subsection{Markov Decision Process}
Figure \ref{fig:rl} illustrates a typical RL scenario. We refer to the learner and decision maker as the \textit{agent}, and the object it interacts with, \ie, everything other than the agent,
is called the \textit{environment}. In RL, the interactions between the agent and the environment are often depicted as a Markov decision process (MDP), which is defined as \cite{bellman1966dynamic}
\begin{equation}
	\mathcal{M}=\langle \mathcal{S},\mathcal{A},\mathcal{T},r,\rho(\bm{s}_{0}),\gamma \rangle,
\end{equation}
where
\begin{itemize}
	\item $\mathcal{S}$ is a state space that can be finite or infinite. We assume that $\mathcal{S}$ is finite or countable infinite for convenience of discussion.
	\item $\mathcal{A}$ is an action space that can be also finite or infinite. A finite action space is often referred to the discrete control tasks, while an infinite action space is often referred to the continuous control tasks.
	\item $\mathcal{T}:\mathcal{S}\times\mathcal{A}\rightarrow\Delta(\mathcal{S})$ is the transition probability, where $\Delta(\mathcal{S})$ is the space of probability distributions over $\mathcal{S}$. $\mathcal{T}(\bm{s}'|\bm{s},\bm{a})$ is the probability of transitioning into state $\bm{s}'$ if an action $\bm{a}$ is taken in state $\bm{s}$. 
	\item $r:\mathcal{S}\times\mathcal{A}\rightarrow\mathbb{R}$ is a reward function. $r(\bm{s},\bm{a})$ is the immediate reward if an action $\bm{a}$ is taken in state $\bm{s}$. 
	\item $\rho(\bm{s}_{0})\in\Delta(\mathcal{S})$ is an initial state distribution to generate the initial state $\bm{s}_{0}$.
	\item $\gamma\in[0,1]$ is a discount factor.
\end{itemize}

For the transition probability, we have:
\begin{equation}
	\sum_{\bm{s}'}\mathcal{T}(\bm{s}'|\bm{s},\bm{a})=1.
\end{equation}

Figure \ref{fig:mdp} illustrates an example of Markov decision process, which has three possible states and two actions.

\begin{figure}[h]
	\centering
	\includegraphics[width=0.85\linewidth]{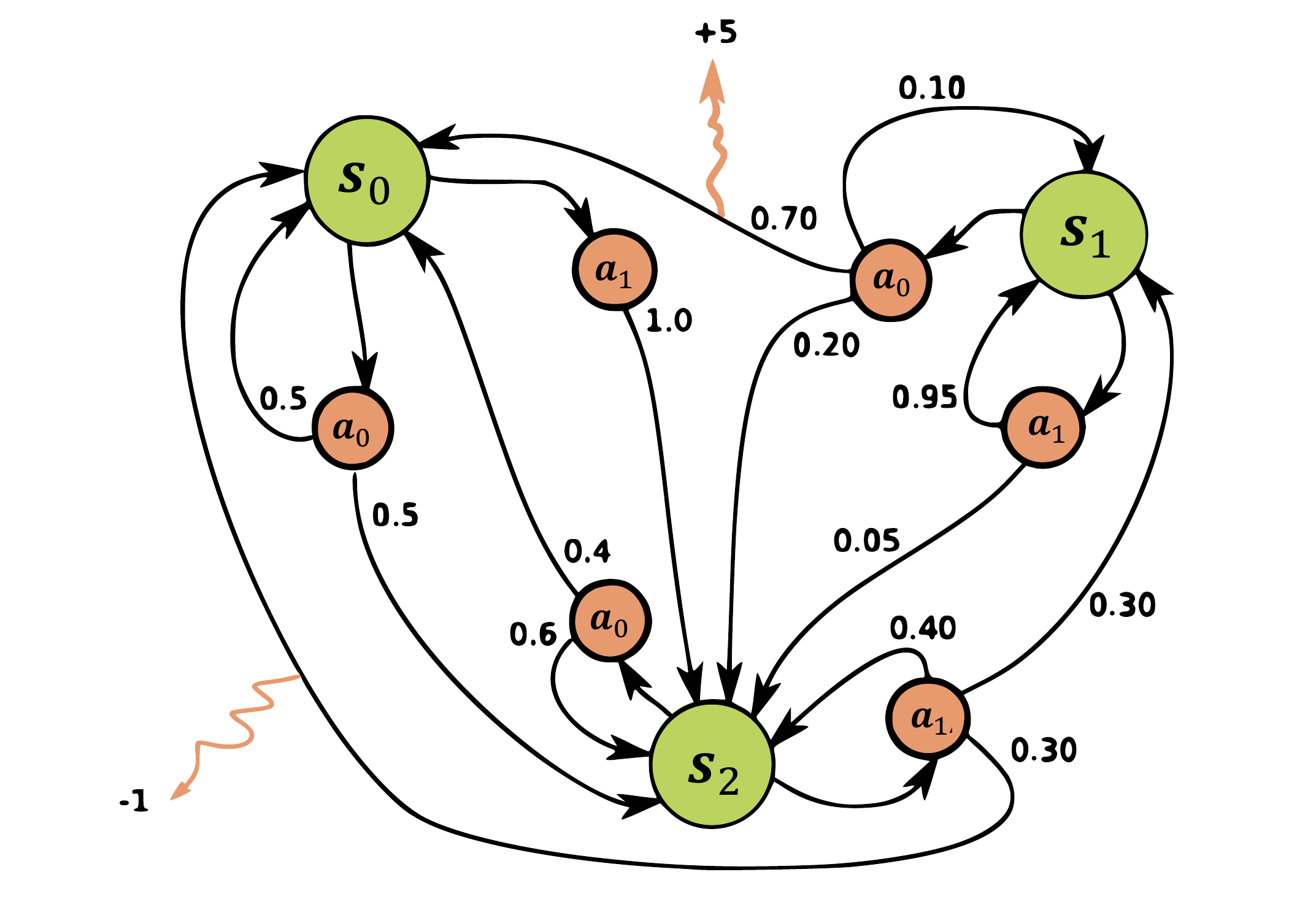}
	\caption{An example of Markov decision process.}
	\label{fig:mdp}
\end{figure}

When the agent takes the action $\bm{a}_{0}$ in state $\bm{s}_1$, the environment has a probability of 0.7 of transitioning into the state $\bm{s}_0$, and the agent will get a reward of $+5$. In contrast, if the agent takes the action $\bm{a}_1$ in state $\bm{s}_2$, the environment has a probability of 0.3 of transitioning into $\bm{s}_0$, and the agent will receive a negative reward of $-1$, which can be considered as a punishment.

\subsubsection{Returns}
So far we have introduced the basic components of a MDP, we next define the learning objective. Assume that a sequence of rewards is received after time step $t$:
\begin{equation}
	R_{t},R_{t+1},R_{t+2},R_{t+3},\dots,
\end{equation}
our objective is to maximize the \textit{expected discounted return}:
\begin{equation}
	G_{t}=R_{t+1}+\gamma R_{t+2}+\gamma^{2}R_{t+3}+\dots=\sum_{k=0}^{\infty}\gamma^{k}R_{t+k+1}.
\end{equation}
In particular, the specific target of the agent is affected by the value of $\gamma$. If $\gamma<1$, the sum has a finite value as long as the reward sequence is bounded. Furthermore, setting $\gamma=0$ leads to $G_t=R_{t+1}$, which effectively restrict the agent to maximize the
immediate reward. Finally, the return objective is more farsighted and thinks more highly of the future rewards when 
$\gamma$ approaches $1$. In particular, returns at successive time steps satisfy:
\begin{equation}
	\begin{aligned}
		G_{t}&=R_{t+1}+\gamma R_{t+2}+\gamma^{2}R_{t+3}+\gamma^{3}R_{t+4}+\dots\\
		&=R_{t+1}+\gamma\left(R_{t+2}+\gamma R_{t+3}+\gamma^{2}R_{t+4}+\dots\right)\\
		&=R_{t+1}+\gamma G_{t+1}.
	\end{aligned}
\end{equation}
This property is very important for the theory and algorithms of RL.

\subsubsection{Policies and Value Functions}
In the most general case, a \textit{policy} specifies a decision strategy in which the agent adaptively selects actions based on the observations history. More specifically, a policy is a mapping from states to probabilities of selecting each possible action. A stationary policy $\pi$ allows the agent to select actions based on the current state, {\em i.e.}, $\bm{a}_{t}\sim\pi(\cdot|\bm{s}_{t})$, while a deterministic and stationary policy is of the form $\pi:\mathcal{S}\rightarrow\mathcal{A}$. 

To evaluate the performance of policies, we define the \textit{value function} to characterize the expected return. Denote by $V^{\pi}(\bm{s})$ the value function of a state $\bm{s}$ under a policy $\pi$, which represents the expected return when starting in $\bm{s}$ and following $\pi$ thereafter. For Markov decision processes, we can formally define $V^{\pi}(\bm{s})$ as 
\begin{equation}
	\begin{aligned}
		V^{\pi}(\bm{s})&=\mathbb{E}_{\pi}\left[ G_{t}|S_{t}=\bm{s} \right]\\
		&=\mathbb{E}_{\pi}\left[ \sum_{k=0}^{\infty}\gamma^{k}R_{t+k+1}\mid S_{t}=\bm{s} \right], \forall \bm{s}\in\mathcal{S},
	\end{aligned}
\end{equation}
where $V^{\pi}$ is called the state-value function for policy $\pi$. 

Similarly, we define the value function $Q^{\pi}(\bm{s},\bm{a})$ of a state-action pair under a policy $\pi$, which represents the expected return when starting from $\bm{s}$, taking the action $\bm{a}$, and thereafter following policy $\pi$:

\begin{equation}
	\begin{aligned}
		Q^{\pi}(\bm{s},\bm{a})&=\mathbb{E}_{\pi}\left[ G_{t}|S_{t}=\bm{s},A_{t}=\bm{a} \right]\\
		&=\mathbb{E}_{\pi}\left[ \sum_{k=0}^{\infty}\gamma^{k}R_{t+k+1}\mid S_{t}=\bm{s},A_{t}=\bm{a} \right], \forall \bm{s}\in\mathcal{S},
	\end{aligned}
\end{equation}
where $Q^{\pi}$ is called the action-value function for policy $\pi$. Next, we introduce a fundamental property of value functions that is widely used in RL and dynamic programming, which satisfies recursive relationships similar to the return definition above. For any policy $\pi$ and any state $\bm{s}$, we have the following Bellman equation:
\begin{equation}
	\begin{aligned}
		V^{\pi}(\bm{s})&=\sum_{\bm a}\pi(\bm{a}|\bm{s})\sum_{\bm{s}^{\prime}}\mathcal{T}\left(\bm{s}^{\prime}|\bm{s},\bm{a}\right)\left[r+\gamma V^{\pi}(\bm{s}')\right],\\
		Q^{\pi}(\bm{s},\bm{a})&=\sum_{\bm{a}'}\pi(\bm{a}'|\bm{s}')\sum_{\bm{s}^{\prime}}\mathcal{T}\left(\bm{s}^{\prime}|\bm{s},\bm{a}\right)\left[r+\gamma Q^{\pi}(\bm{s}',\bm{a}')\right].
	\end{aligned}
\end{equation}
The two equations will be used to design a popular RL algorithm in the following sections.

\subsubsection{Optimal Policies and Optimal Value Functions}
Given two policies $\pi$ and $\pi'$, we say $\pi$ is better than or equal to $\pi'$ if and only if its expected return is greater than or equal to that of $\pi'$ for all states, {\em i.e.} $\pi\succeq\pi'$ if and only if
\begin{equation}
	V^{\pi}(\bm{s})\geq V^{\pi\prime}(\bm{s}),\forall \bm{s}\in\mathcal{S}.
\end{equation}
We denote by $\pi^{*}$ the optimal policy. Despite the fact that multiple optimal policies may exist, they share the same state-value function, called the optimal state-value function defined as
\begin{equation}
	V^{*}(\bm{s})=\max_{\pi}V^{\pi}(\bm{s}),\forall\bm{s}\in\mathcal{S}.
\end{equation}
They also share the same optimal action-value function, denoted $Q^{*}$, and defined as
\begin{equation}
	Q^{*}(\bm{s},\bm{a})=\max_{\pi}Q^{\pi}(\bm{s},\bm{a}),
\end{equation}
for all $\bm{s}\in\mathcal{S}$ and $\bm{a}\in\mathcal{A}$. To find the optimal policy, a trivial approach is to exhaustively search all the possible policies and compare the resulting expected returns. However, such an approach will incur prohibitive computational complexity. To address this problem, two representative RL algorithms have been introduced in the literature.

\subsection{Optimal Value Algorithms}
\subsubsection{Q-learning}
Q-learning is a value-based, model-free and off-policy RL algorithm \cite{watkins1992q}. An off-policy learner learns the value of the optimal policy that is independently of the actions of the agent. We first introduce a Bellman optimality equation before detailing the Q-learning. After substituting the optimal action-value function into the Bellman equation, we have:
\begin{equation}
	\begin{aligned}
		Q^{*}(\bm{s},\bm{a})&=\sum_{\bm{a}'}\pi(\bm{a}'|\bm{s}')\sum_{\bm{s}^{\prime}}\mathcal{T}\left(\bm{s}^{\prime}|\bm{s},\bm{a}\right)\left[r+\gamma Q^{*}(\bm{s}',\bm{a}')\right]\\
		&=\sum_{\bm{s}^{\prime}}\mathcal{T}\left(\bm{s}^{\prime}|\bm{s},\bm{a}\right)\left[r+\gamma \max_{\bm{a}'}\{Q^{*}(\bm{s}',\bm{a}')\}\right]
	\end{aligned}
\end{equation}
Therefore, Q-learning defines the following iterative formula:
\begin{equation}
	Q\left(S_{t}, A_{t}\right) \leftarrow Q\left(S_{t}, A_{t}\right)+\alpha\left[R_{t+1}+\gamma \max _{\bm a} Q\left(S_{t+1}, \bm{a}\right)-Q\left(S_{t}, A_{t}\right)\right],
\end{equation}
where $\alpha$ is a step size. In this formula, the learned action-value function straightforwardly approximates $Q^{*}$, which is independent of the policy being followed. Such a setting significantly simplifies the analysis of the algorithm and convergence proofs. It is worth mentioning that the policy still has an effect because it determines which state-action pairs will be visited and updated. Furthermore, the convergence condition of Q-learning is to visit all possible state-action pairs infinitely. Finally, we summarize the detailed workflow of the Q-learning in Algorithm \ref{algo:ql}.

\begin{algorithm}[h]
	\caption{Q-learning for estimating $\pi\approx\pi^{*}$}
	\label{algo:ql}
	\begin{algorithmic}[1]
		\STATE Initialize $Q(\bm{s},\bm{a})$ for all $\bm{s}\in\mathcal{S},\bm{a}\in\mathcal{A}$ randomly, except $Q({\rm terminal},\cdot)=0$;
		\STATE Initialize a step size $\alpha\in(0,1]$ and a small $\epsilon>0$;
		\STATE Loop for each episode:
		\STATE $\quad$ Initialize $S$;
		\STATE $\quad$ Loop for each step of episode:
		\STATE $\quad\quad$ Choose $A$ from $S$ using policy derived from $Q$ ({\em e.g.}, "$\epsilon$-greedy");
		\STATE $\quad\quad$ Take action $A$, observe $R$, $S'$;
		\STATE $\quad\quad$ $Q\left(S_{t}, A_{t}\right) \leftarrow Q\left(S_{t}, A_{t}\right)+\alpha\left[R_{t+1}+\gamma \max _{\bm a} Q\left(S_{t+1}, \bm{a}\right)-Q\left(S_{t}, A_{t}\right)\right]$;
		\STATE $\quad\quad$ $S\leftarrow S'$;
		\STATE $\quad$ until $S$ is terminal.
	\end{algorithmic}
\end{algorithm}

\subsubsection{Deep Q-learning}
When handling complex environments with high-dimensional observations, it is critical to utilize function approximation as a compact representation of action values \cite{mnih2013playing}. We represent the action-value function $Q(\bm{s},\bm{a},\bm\theta)$ by a neural network with parameters $\bm{\theta}$. Thus, the optimization objective is to find $\bm{\theta}$ such that $Q(\bm{s},\bm{a},\bm\theta)\approx Q^{*}(\bm{s},\bm{a})$. Therefore, the optimal parameters can be approximated by minimizing the squared temporal difference error:
\begin{equation}
	L(\bm{\theta})=\left[r+\gamma\max_{\bm{a}'}Q(\bm{s}',\bm{a}',\bm{\theta})-Q(\bm{s},\bm{a},\bm{\theta})\right]^{2}.
\end{equation}

Despite its simplicity, Q-learning suffers from low efficiency and inadequate convergence when handling complex environments like Atari games.

\subsection{Policy Gradient Algorithms}
The optimal value algorithms learn the optimal policies by estimating the optimal action-value function. Alternatively, it is feasible to omit the estimation procedure and straightforwardly approximate the optimal policies via a parameterized function, and update its parameters iteratively. Since the iterative procedure corresponds to the gradient of policy, such algorithms are called as policy gradient algorithms.

\subsubsection{Policy Gradient Theorem}
Recalling the identity of the stationary policy, it satisfies
\begin{equation}
	\sum_{\bm a}\pi(\bm{a}|\bm{s})=1,
\end{equation}
for all $\bm{s}\in\mathcal{S}$. We denote by $\pi(\bm{a}|\bm{s},\bm{\theta})$ the parameterized policy, it should also holds the normalization condition and permits differentiability \cite{agarwal2020reinforcement}. Therefore, we introduce an action preference function $h(\bm{s},\bm{a},\bm{\theta})$ and define the following policies:
\begin{equation}
	\pi(\bm{a}|\bm{s},\bm{\theta})=\frac{\exp\{h(\bm{s},\bm{a},\bm{\theta})\}}{\sum_{\bm{a}'\in\mathcal{A}}\exp\{h(\bm{s},\bm{a}',\bm{\theta})\}},
\end{equation}
where $h$ may have multiple forms, such as linear combinations, neural network and so on. Next, we introduce the policy gradient theorem as follows:
\begin{theorem}\label{thm:pg}
	For episodic case, denote by $J(\bm{\theta})=V^{\pi}(\bm{s}_{0})$ the expected return of policy $\pi$, the gradients of $J(\bm{\theta})$ with respect to $\bm{\theta}$ can be expressed as
	\begin{equation}
		\begin{aligned}
			\nabla J(\bm{\theta})\propto \sum_{\bm s}\mu(\bm{s})\sum_{\bm a}Q^{\pi}(\bm{s},\bm{a})\nabla \pi(\bm{a}|\bm{s}),
		\end{aligned}
	\end{equation}
	where $\mu$ is the on-policy distribution under policy $\pi$.
\end{theorem}
\begin{proof}
	See proof in \cite{sutton2018reinforcement}.
\end{proof}

Theorem \ref{thm:pg} indicates that we can obtains the gradients of the expected return as long as the parameters of the policy are given. As a result, we can update $\bm{\theta}$ following the gradients to increase the expected return. Note that the right side of the policy gradient is a sum over states weighted by the frequency of the states occurrence under the target policy. Thus, the policy gradient can be expressed as

\begin{equation}
	\begin{aligned}
		\nabla J(\bm{\theta})&\propto \sum_{\bm s}\mu(\bm{s})\sum_{\bm a}Q^{\pi}(\bm{s},\bm{a})\nabla \pi(\bm{a}|\bm{s}),\\
		&=\mathbb{E}_{\pi}\bigg[ \sum_{\bm a}Q^{\pi}(S_{t},\bm{a})\nabla \pi(\bm{a}|S_{t},\bm{\theta}) \bigg]\\
		&=\mathbb{E}_{\pi}\left[G_{t} \frac{\nabla \pi\left(A_{t}|S_{t},\bm{\theta}\right)}{\pi\left(A_{t}|S_{t}, \bm{\theta}\right)}\right]\\
		&=\mathbb{E}_{\pi}\left[G_{t}\nabla\ln\pi\left(A_{t}|S_{t},\bm{\theta}\right)\right].
	\end{aligned}
\end{equation}

Therefore, we can derive the following vanilla policy gradient (VPG) algorithm \cite{williams1992simple}:
\begin{algorithm}
	\caption{Vanilla Policy Gradient}
	\label{algo:vpg}
	\begin{algorithmic}
		\STATE Randomly initialize the policy $\pi$ using parameters $\bm{\theta}$;
		\STATE Initialize a step size $\alpha$;
		\STATE Loop for each episode:
		\STATE $\quad$ Generate an episode $S_{0},A_{0},R_{1},\dots,S_{T-1},A_{T-1},R_{T}$, following $\pi_{\bm\theta}$;
		\STATE $\quad$ Loop for each step of the episode $t=0,1,\dots,T-1$:
		\STATE $\quad\quad$ $G \leftarrow \sum_{k=t+1}^{T} \gamma^{k-t-1} R_{k}$;
		\STATE $\quad\quad$ $\bm{\theta} \leftarrow \boldsymbol{\theta}+\alpha \gamma^{t} G \nabla \ln \pi\left(A_{t}|S_{t}, \bm{\theta}\right)$.
	\end{algorithmic}
\end{algorithm}

\subsubsection{Proximal Policy Optimization}
In Algorithm \ref{algo:vpg}, we use Monte-Carlo sampling to calculate the return $G$, and the policy can only be updated at the end of the episode. To promote the learning efficiency, it is reasonable to leverage a parameterized function to approximate the value function and estimate the expected return. After that, the estimated return is used to compute the policy gradients and update the policy parameters. Such a setting is called \textit{actor-critic} method. In particular, an efficient algorithm entitled proximal policy optimization (PPO) was developed in \cite{schulman2017proximal}. PPO is a model-free, on-policy, and actor-critic RL algorithm .

We consider updating the policy $\pi(\bm{\theta})$ iteratively. Suppose that a new policy $\pi(\bm{\theta}_{k})$ is derived at time step $k$, the following equality holds:
\begin{equation}
	\mathbb{E}_{\pi(\bm{\theta})}\left[G_{0}\right]=\mathbb{E}_{\pi\left(\bm{\theta}_{k}\right)}\left[G_{0}\right]+\mathbb{E}_{\pi(\bm{\theta})}\left[\sum_{t=0}^{+\infty} \gamma^{t} Z_{\pi\left(\bm{\theta}_{k}\right)}\left(S_{t}, A_{t}\right)\right],
\end{equation}
where $Z_{\pi}=Q^{\pi}-V^{\pi}$ is the advantage function. To maximize $\mathbb{E}_{\pi(\bm{\theta})}\left[G_{0}\right]$, it suffices to maximize 
\begin{equation}
	\mathbb{E}_{\pi(\bm{\theta})}\left[\sum_{t=0}^{+\infty} \gamma^{t} Z_{\pi\left(\bm{\theta}_{k}\right)}\left(S_{t}, A_{t}\right)\right].
\end{equation}
By applying importance sampling, we have:
\begin{equation}
	\mathbb{E}_{S_{t}, A_{t} \sim \pi(\bm{\theta})}\left[Z_{\pi(\bm{\theta}_{k})}\left(S_{t}, A_{t}\right)\right]=\mathbb{E}_{S_{t} \sim \pi(\bm{\theta}), A \sim \pi\left(\bm{\theta}_{k}\right)}\left[\frac{\pi\left(A_{t}|S_{t}, \bm{\theta}\right)}{\pi\left(A_{t}|S_{t},\bm{\theta}_{k}\right)} Z_{\pi\left(\bm{\theta}_{k}\right)}\left(S_{t}, A_{t}\right)\right]
\end{equation}
However, it is difficult to calculate the expectation with respect to $S_{t}\sim\pi(\bm\theta)$. To address this problem, we approximate the expectation of $S_{t}\sim\pi(\bm\theta)$ as $S_{t}\sim\pi(\bm{\theta}_{k})$, which is called surrogate advantage, {\em i.e.}, 

\begin{equation}
	\mathbb{E}_{S_{t}, A_{t} \sim \pi(\boldsymbol{\theta})}\left[Z_{\pi(\bm{\theta}_{k})}\left(S_{t}, A_{t}\right)\right] \approx \mathbb{E}_{S_{t}, A_{t}\sim \pi\left(\bm{\theta}_{k}\right)}\left[\frac{\pi\left(A_{t}|S_{t},\bm{\theta}\right)}{\pi\left(A_{t}|S_{t},\bm{\theta}_{k}\right)} Z_{\pi\left(\bm{\theta}_{k}\right)}\left(S_{t}, A_{t}\right)\right].
\end{equation}
Therefore, we obtain the approximation of $\mathbb{E}_{\pi(\bm{\theta})}\left[G_{0}\right]$:
\begin{equation}
	L(\bm{\theta})=\mathbb{E}_{\pi\left(\bm{\theta}_{k}\right)}\left[G_{0}\right]+\mathbb{E}_{S_{t}, A_{t}\sim\pi\left(\bm{\theta}_{k}\right)}\left[\sum_{t=0}^{+\infty}\gamma^{t}\frac{\pi\left(A_{t}|S_{t},\bm{\theta}\right)}{\pi\left(A_{t}|S_{t},\bm{\theta}_{k}\right)}Z_{\pi\left(\bm{\theta}_{k}\right)}\left(S_{t}, A_{t}\right)\right].
\end{equation}
It is straightforward to show that $\mathbb{E}_{\pi(\bm{\theta})}\left[G_{0}\right]$ and $L(\bm{\theta})$ have same gradients when $\bm{\theta}=\bm{\theta}_{k}$. Thus, it is possible to learn better policies via optimizing the surrogate advantage. PPO redesigned the optimization objective as 
\begin{equation}
	\mathbb{E}_{\pi(\bm{\theta}_{k})}\left[\min\left(\frac{\pi\left(A_{t}|S_{t},\bm{\theta}\right)}{\pi\left(A_{t}|S_{t},\bm{\theta}_{k}\right)}Z_{\pi\left(\bm{\theta}_{k}\right)},Z_{\pi\left(\bm{\theta}_{k}\right)}+\epsilon|Z_{\pi\left(\bm{\theta}_{k}\right)}|\right) \right],
\end{equation}
where $\epsilon\in(0,1)$ is a weighting coefficient. This objective can stabilize the update process by controlling the gap between the new policy and the old policy. Moreover, PPO is very efficient and easy to implement, serving as the baseline in many RL experiments.

%
%
%
%
%
%

\section{Classical Exploration Strategies}
\label{chap:two}

\subsection{$K$-Armed Bandit Problem}

In Chapter \ref{chap:one}, the convergence condition of RL algorithms requires visiting all possible
state-action pairs infinitely. In other words, the agent need to fully explore the environment to find the optimal policies. However, there is a critical tradeoff between the exploitation and exploration in RL, which can be shown via the $K$-armed bandit problem. As shown in Figure~\ref{fig:mab}, a gambler can choose to press one of the arms after putting in a coin. Each arm will spit out a coin with a certain probability, but this probability is unknown for the gambler. The goal of gamblers is to maximize their return through certain policies, {\em i.e.}, to get the most coins.

\begin{figure}[h]
	\centering
	\includegraphics[width=0.85\linewidth]{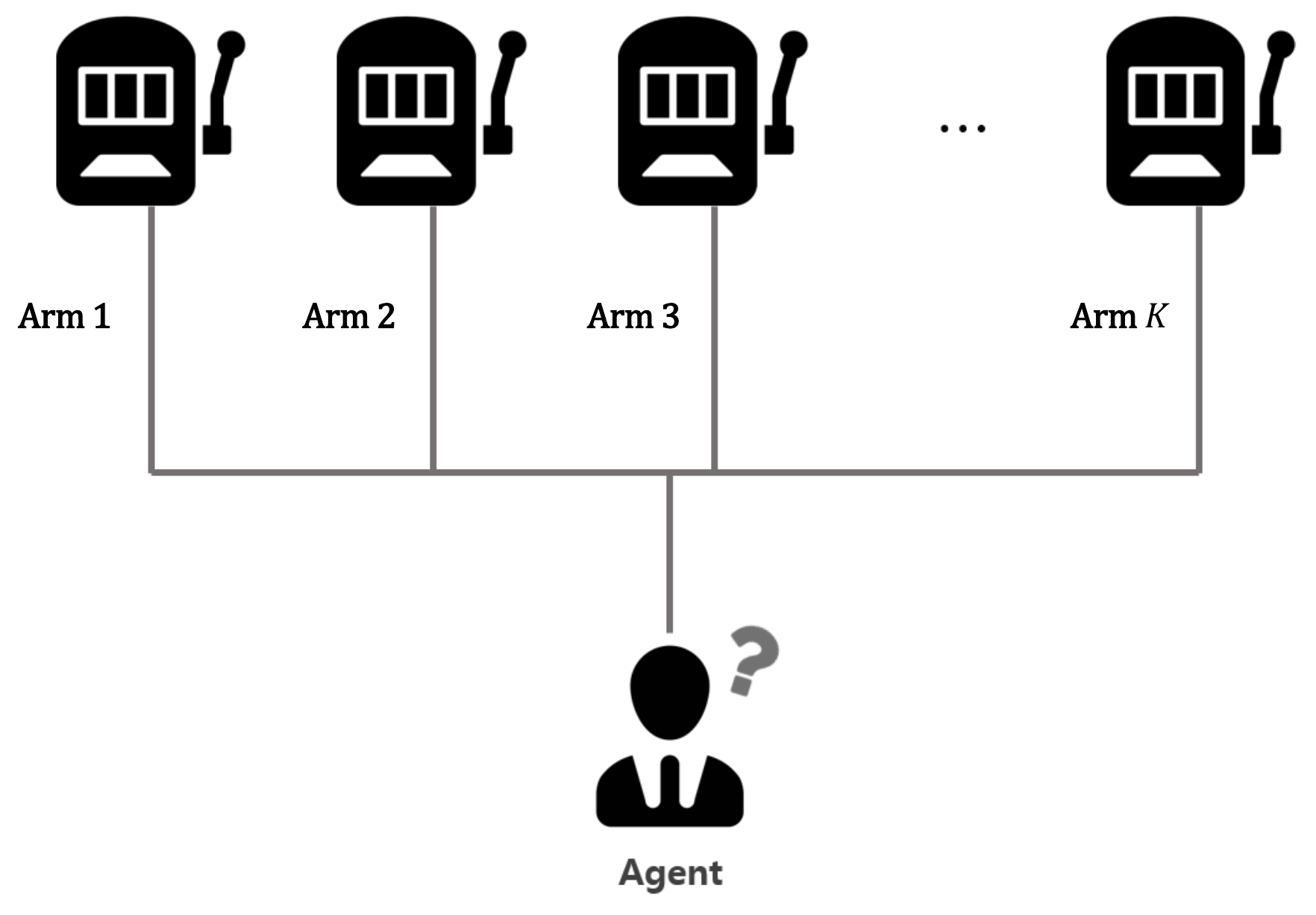}
	\caption{A $K$-armed bandit.}
	\label{fig:mab}
\end{figure}

If we want to know the expected return of each arm, we can distribute the trial opportunities equally to all the arms, {\em i.e.}, the exploration-only method. Based on the
outcome, we can estimate the expected return of each arm using the average payout probability. In contrast, if we want to execute the action with maximum reward, it suffices to apply the exploitation-only method, {\em i.e.}, to press the current optimal arm. Therefore, the exploration-only method can estimate the return of each arm accurately
at the cost of the large resources spent on identifying the optimal arm. In contrast, the exploitation-only method only focus the current return without spending further efforts in finding the true optimal arm, which is shortsighted. Both these two methods cannot effectively maximize the long-term rewards. Since the number of total attempts is finite, a comprehensive algorithm must be found by striking an appropriate balance between exploitation and exploration.

\subsection{Soft Policy}
To cope with the problem above, a well-known method entitled $\epsilon$-greedy algorithm was proposed in the literature. In on-policy control, we say that a policy is generally \textit{soft}, if it holds
\begin{equation}
	\pi(\bm{a}|\bm{s})> 0,
\end{equation}
for all $\bm{s}\in\mathcal{S}$ and all $\bm{a}\in\mathcal{A}$. A soft policy can select all possible actions, which allows the agent to visit more possible states and state-action pairs. The $\epsilon$-greedy policy is a special case of the soft policies. It selects an action corresponding to the maximal estimated action value with probability $1-\epsilon$, while any random action with probability $\epsilon$. More specifically, it follows
\begin{equation}
	\pi\left(\bm{a}|S_{t}\right) \leftarrow 
	\begin{cases}1-\epsilon+\epsilon /\left|\mathcal{A}\left(S_{t}\right)\right| & \text { if } \bm{a}=A^{*}, \\ 
		\epsilon /\left|\mathcal{A}\left(S_{t}\right)\right| & \text { if } \bm{a} \neq A^{*},
	\end{cases}
\end{equation}
where $A^{*}\leftarrow\underset{\bm{a}}{\rm argmax}\:Q(S_{t},\bm{a})$. Moreover, we have the following theorem:
\begin{theorem}
	The	$\epsilon$-greedy policy satisfies the policy improvement theorem.
\end{theorem}
\begin{proof}
	See proof in \cite{sutton2018reinforcement}.
\end{proof}
Now we are ready to design the $\epsilon$-greedy algorithm for $K$-armed bandit problem. Let $Q_{n}(k)$ denote the average return of the $k$-th arm after the $n$-th attempt, we have the following update formula:
\begin{equation}
	Q_{n}(k)=\frac{1}{n}\left((n-1) \times Q_{n-1}(k)+v_{n}\right),
\end{equation}
where $v_{n}$ denotes the reward of the $n$-th attempt. Finally, we summarize the $\epsilon$-greedy algorithm for $K$-armed bandit problem in Algorithm \ref{algo:ega}. Note that $Q(i)$ denotes the average return of the $i$-th arm, $\mathrm{count}(i)$ denotes the number of attempts, and ${\rm rand}()$ will randomly generates a number from $[0,1]$.

\begin{algorithm}[h]
	\caption{$\epsilon$-greedy Algorithm}
	\label{algo:ega}
	\begin{algorithmic}[1]
		\STATE \textbf{Input}: Number of arms $K$, a reward function $r$, number of attempts $T$, exploration probability $\epsilon$;
		\STATE Set $G=0$;
		\STATE Initialize $Q(i)=0, \mathrm{count}(i)=0$ for $i=1,2,\dots,K$;
		\FOR {$t=1,2,\dots,T$}
		\IF {$\mathrm{rand}()<\epsilon$}
		\STATE Randomly select a $k$ from $1,2,\dots,K$;
		\ELSE
		\STATE $k=\underset{i}{\rm argmax}\:Q(i)$;
		\ENDIF
		\STATE $v=r(k)$;
		\STATE $G=G+v$;
		\STATE $Q(k)=\frac{Q(k)\cdot\mathrm{count}(k)+v}{\mathrm{count}(k)+1}$;
		\STATE $\mathrm{count}(k)=\mathrm{count}(k)+1$;
		\ENDFOR
		\STATE \textbf{Output}: The accumulative rewards $G$.
	\end{algorithmic}
\end{algorithm}

When the reward distribution has heavier tails, a larger $\epsilon$ is needed to encourage more exploration. Furthermore, for a sufficiently large number of attempts, all the arm rewards can be accurately estimated. As a result, no further exploration is required. Therefore, it suffices to make $\epsilon$ decay with attempts, such as $\epsilon=1/\sqrt{t}$.

\subsection{Upper Confidence Bound}
The $\epsilon$-greedy method selects non-greedy actions with equal probability, regardless of their estimated returns derived from the past attempts. However, it is more reasonable to choose actions based on their estimated returns by taking into account both their expected return and the uncertainty associated with these estimates. Inspired by this observation, the upper confidence bound (UCB) algorithm selects actions by \cite{garivier2011upper}
\begin{equation}
	k=\underset{i}{\rm argmax}\:\big[Q(i)+c\sqrt{\frac{\ln t}{N_{t}(i)}}\big],
\end{equation}
where $N_{t}(i)$ is the number of times that the $i$-th arm has been selected prior to time $t$, and $c>0$ controls the degree of exploration. The square-root term measures the uncertainty in the return estimation of the $i$-th arm. The uncertainty term decreases when $N_{t}(i)$ increases. On the opposite, if the $i$-th arm is not selected but $t$ increases, the uncertainty estimate increases. As a result, all arms will eventually be selected with arms of either low estimated return or high selection frequency being selected with decaying frequency.

\subsection{Thompson Sampling}
In Thompson sampling (TS), the agent tracks a belief over the probability distribution of the optimal arms and samples from this distribution, which is shown to solve the $K$-armed bandit problem more effectively \cite{russo2017tutorial}. More specifically, TS assumes that $Q(i)$ follows a Beta distribution that is a family of continuous probability distributions defined on the interval $[0, 1]$ parameterized by two positive shape parameters, namely $\alpha$ and $\beta$. At each time step $t$, TS samples an expected reward $\tilde{Q}(i)$ from the prior distribution $\mathrm{Beta}(\alpha_{i},\beta_{i})$ for every arm selection. After that, the action is chosen by 
\begin{equation}
	k_{TS}=\underset{i}{\rm argmax}\:\tilde{Q}(i).
\end{equation}
After the true reward is observed, the Beta distribution is updated following 
\begin{equation}
	\begin{aligned}
		\alpha_{i}&\leftarrow \alpha_{i}+r(k)\mathbbm{1}[k_{TS}=i],\\
		\beta_{i}&\leftarrow \beta_{i}+(1-r(k))\mathbbm{1}[k_{TS}=i],
	\end{aligned}
\end{equation}
where $\mathbbm{1}$ is the indicator function. Since TS samples reward estimations from prior distributions and the reward probability corresponding to each action is currently considered to be the optimal, it actually implements the idea of probability matching. 

\subsection{Boltzmann Exploration}
Next, we review a more straightforward but efficient exploration method entitled the Boltzmann Exploration. The Boltzmann exploration compromises the exploitation and exploration based on the known average return. The key insight of the Boltzmann exploration is to give higher selection probability to arms with higher average returns. The Boltzmann distribution is defined as \cite{reichl1999modern}
\begin{equation}\label{eq:boltzmann exploration}
	P(k)=\frac{e^{\frac{Q(k)}{\tau}}}{\sum_{i=1}^{K} e^{\frac{Q(i)}{\tau}}},
\end{equation}
where $Q(i)$ is the average return of the $i$-th arm, $\tau>0$ is a coefficient called as \textit{temperature}. The preference for exploration increases with the value of $\tau$. 
\begin{figure}[h!]
	\centering
	\includegraphics[width=0.75\linewidth]{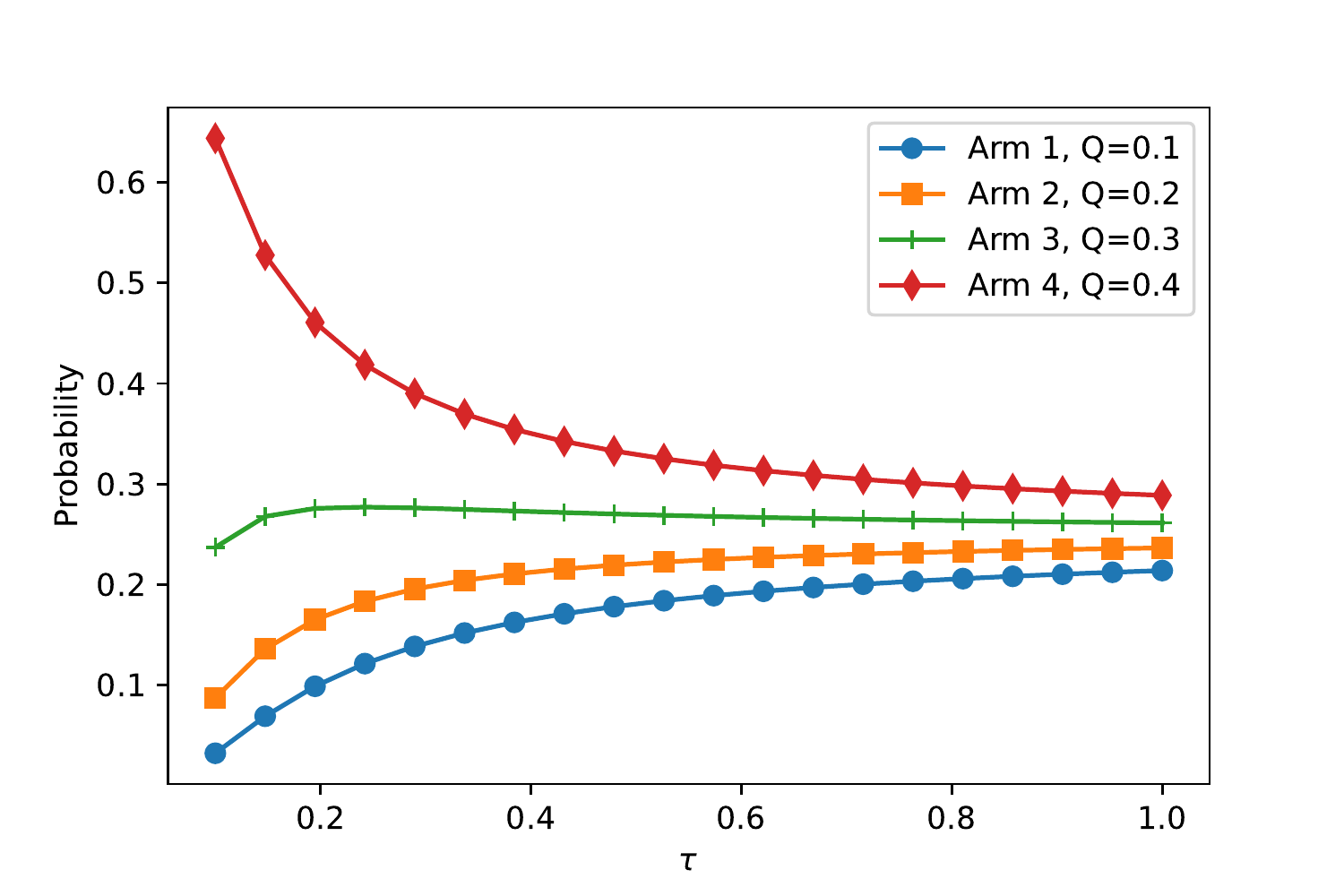}
	\caption{Selection probabilities versus the value of $\tau$.}
	\label{fig:boltzmann}
\end{figure}

For illustration purposes, We consider an example of four arms with average return $0.1$, $0.2$, $0.3$ and $0.4$, Figure~\ref{fig:boltzmann} illustrates the relation between the selection probabilities and $\tau$. Inspection of Figure~\ref{fig:boltzmann} suggests that the selection probabilities tend to be equal when the temperature increases. Equipped with the Boltzmann distribution, we summarize the detailed exploration algorithm in Algorithm \ref{algo:Boltzmann exploration}. 

\begin{algorithm}[h!]
	\caption{Boltzmann Exploration Algorithm}
	\label{algo:Boltzmann exploration}
	\begin{algorithmic}[1]
		\STATE \textbf{Input}: Number of arms $K$, a reward function $r$, number of attempts $T$, temperature coefficient $\tau$;
		\STATE Set $G=0$;
		\STATE Initialize $Q(i)=0, \mathrm{count}(i)=0$ for $i=1,2,\dots,K$;
		\FOR {$t=1,2,\dots,T$}
		\STATE Sample $k$ based on $P(k)$;
		\STATE $v=r(k)$;
		\STATE $G=G+v$;
		\STATE $Q(k)=\frac{Q(k)\cdot\mathrm{count}(k)+v}{\mathrm{count}(k)+1}$;
		\STATE $\mathrm{count}(k)=\mathrm{count}(k)+1$;
		\ENDFOR
		\STATE \textbf{Output}: The accumulative rewards $G$.
	\end{algorithmic}
\end{algorithm}

\subsection{Action Entropy Maximization}\label{sect:aem}
Exploration methods like $\epsilon$-greedy policy and Boltzmann exploration are prone to eventually learn the optimal policy in tabular setting. However, they are inefficient and may be futile when handling complex environments with high-dimensional observations. To address the exploration problem in deep RL, an effective method is action entropy maximization. In contrast to the aforementioned methods, action entropy maximization is a passive exploration method that adds an entropy term into
the loss function while utilizing neural networks for function approximation.

We first formally define the Shannon entropy as follows \cite{shannon2001mathematical}:
\begin{definition}\label{def:shannon entropy}
	Let $X\in\mathbb{R}^{m}$ be a random vector that has a density function $f(\bm{x})$ with respect to Lebesgue measure on $\mathbb{R}^{m}$, and let $\mathcal{X}=\{\bm{x}\in\mathbb{R}^{m}:f(\bm{x})>0\}$ be the support of the distribution. The Shannon entropy is defined as:
	\begin{equation}\label{eq:hf}
		H(f)=-\int_{\mathcal{X}}f(\bm{x})\log f(\bm{x})d \bm{x}.
	\end{equation}
\end{definition}

Equipped with Definition \ref{def:shannon entropy}, the action entropy is calculated as
\begin{equation}
	H(\pi(\cdot|\bm{s}))=-\int_{\mathcal{A}} \pi(\bm{a}|\bm{s})\log \pi(\bm{a}|\bm{s})d\bm{a}.
\end{equation}
For discrete action space, it has a simpler form:
\begin{equation}
	H(\pi(\cdot|\bm{s}))=\sum_{\bm a\in\mathcal{A}}-\pi(\bm{a}|\bm{s})\log \pi(\bm{a}|\bm{s}).
\end{equation}
Similar to Boltzmann exploration, action entropy evaluates the preference of exploration and exploitation. For instance, let $|\mathcal{A}|=4,\pi(\cdot|\bm{s}_{0})=\{0.1,0.2,0.3,0.4\}$, then we have
\begin{equation}
	H(\pi(\cdot|\bm{s}_{0}))=1.28.
\end{equation}
Let $\pi(\cdot|\bm{s}_{1})=\{0.25,0.25,0.25,0.25\}$, then
\begin{equation}
	H(\pi(\cdot|\bm{s}_{1}))=1.39.
\end{equation}
It is natural to find that the decentralized action probabilities produce higher action entropy, which encourages the agent to visit the action space more evenly. Therefore, it is feasible to utilize the action entropy as an additional reward, and redefine the objective of RL as
\begin{equation}\label{eq:merl}
	\pi^{*}=\underset{\pi\in\Pi}{\rm argmax}\:\mathbb{E}_{\tau\sim\pi}\left[\sum_{t=0}^{T-1}\gamma^{t}r(\bm{s}_{t},\bm{a}_{t})+\alpha H(\pi(\cdot|\bm{s}_{t}))\right],
\end{equation}
where $\Pi$ is the set of all stationary polices, $\tau=(\bm{s}_{0}, \bm{a}_{0}, \dots, \bm{a}_{T-1}, \bm{s}_{T})$ is the trajectory collected by the agent, and $\alpha$ is the temperature parameter that determines the importance of the entropy item \cite{haarnoja2018soft}. Moreover, we have the following convergence theorem:
\begin{theorem}
	Repeated application of	soft policy evaluation and soft policy improvement from any $\pi\in\Pi$ can converge to the optimal policy $\pi^{*}$.
\end{theorem}
\begin{proof}
	See proof in \cite{haarnoja2017reinforcement}.
\end{proof}
A well-known algorithm for solving Eq.~\eqref{eq:merl} is soft actor-critic that can be found in \cite{haarnoja2018soft}. Note that the entropy regularizer can be used in conjunction with any RL algorithms.

\subsection{Noise-Based Exploration}
Another representative exploration method is the noise-based exploration, which adds noise into observation, action and even parameter space. Active exploration methods like $\epsilon$-greedy policy and Boltzmann exploration are often used for discrete action spaces. In fact, the $\epsilon$-greedy policy can also be used in continuous control tasks, {\em i.e.}, to randomly sample an action with probability $\epsilon$ in each time step. In \cite{lillicrap2015continuous}, a deep deterministic policy gradient (DDPG) is proposed that solves the problem of exploration independently from the learning algorithm. DDPG constructs an exploration policy by adding noise sampled from a noise process:
\begin{equation}
	\pi'(\bm{s})=\pi(\bm{s},\bm{\theta})+\mathcal{N},
\end{equation}
where $\pi$ is a deterministic policy and $\mathcal{N}$ can be chosen to adapt the environment. For instance, we can use Ornstein-Uhlenbeck process (OUP) \cite{uhlenbeck1930theory} or Gaussian process.

In contrast, \cite{plappert2017parameter} proposes to improve exploration via parameter space noise. As shown in Figure \ref{fig:pnoise}, parameter space noise directly injects randomness into the parameters of the agent, disturbing the types of decisions it makes such that they always fully depend on what the agent currently senses. 

\begin{figure}[h]
	\centering
	\includegraphics[width=0.85\linewidth]{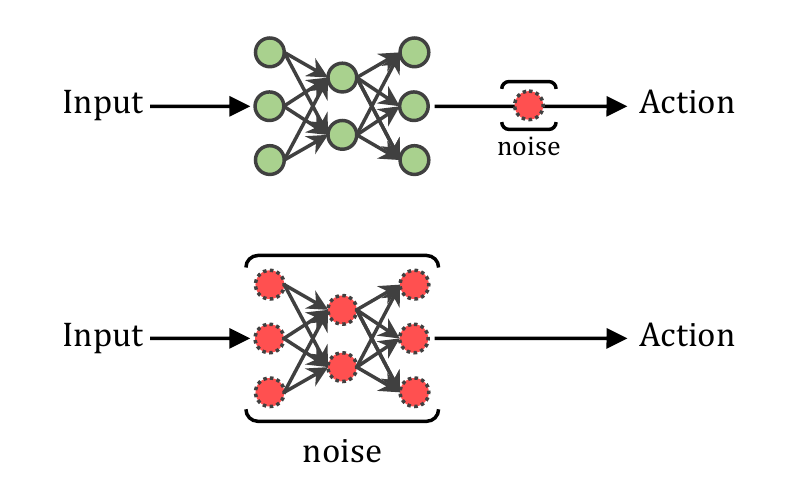}
	\caption{Action space noise (top), compared to parameter space noise (bottom).}
	\label{fig:pnoise}
\end{figure}

Denote by $\bm{\theta}$ the parameters of the agent, then the perturbed parameters $\tilde{\bm{\theta}}$ are computed as
\begin{equation}
	\tilde{\bm{\theta}}=\bm{\theta}+\mathcal{N},
\end{equation}
where $\mathcal{N}$ is a noise process, such as Gaussian noise. In particular, an adaptive variance $\sigma$ is defined as
\begin{equation}
	\sigma_{k+1}= \begin{cases}\alpha \sigma_{k} & \text { if } d(\pi, \tilde{\pi}) \leq \delta \\ \frac{1}{\alpha} \sigma_{k} & \text { otherwise }\end{cases},
\end{equation}
where $\alpha\in\mathbb{R}_{+}$ is a scaling factor, $\delta\in\mathbb{R}_{+}$ is a threshold value, and $d(\cdot,\cdot)$ is a distance measure that depends on the concrete algorithm. Figure \ref{fig:pnoise_return} illustrates the performance comparison with action-noise methods.

\begin{figure}[h]
	\centering
	\includegraphics[width=\linewidth]{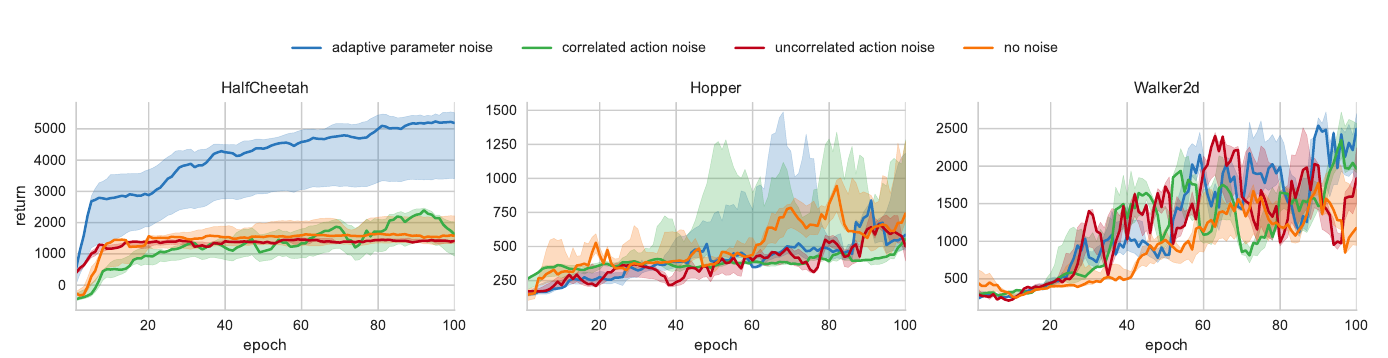}
	\caption{Median DDPG returns for continuous control environments plotted over epochs.}
	\label{fig:pnoise_return}
\end{figure}

It is interesting to find in Figure \ref{fig:pnoise_return} that parameter noise actually helped algorithms explore their environments more effectively, leading to higher scores. This is because adding noise to the parameters will make the exploration consistent across episodes, while adding noise to the action space will lead to unpredictable exploration.

\section{Intrinsically-Motivated Exploration}
\label{chap:three}

\subsection{Intrinsic Motivation of Learning}
Most exploration methods in Chapter \ref{chap:two} promotes the exploration of action space, which encourages the agent to visit more unknown state-action pairs to learn better policies. However, these techniques can eventually learn the optimal policy in the tabular setting, but they are inefficient when handling complex environments with high-dimensional observations \cite{pathak2017curiosity}. Despite the advantages of the action entropy maximization and noise-based exploration, they cannot significantly improve the critical exploration problem of state space. To address this problem, \cite{white1959motivation} analyzed the learning motivation of the agent, which can be distinguished as extrinsic motivation and intrinsic motivation. Extrinsic motivation refers to being moved to do something because of some specific rewarding mechanism. For instance, we want to get more points in Atari games or more grades in examination. In contrast, intrinsic motivation refers to being moved to do something because it is inherently pleasant. When an infant plays, waves its arms or hums, he has no explicit teacher. Most trappings of modern life like college and a good job are so far into the future, which provides no effective reinforcement signal. Therefore, intrinsic motivation leads agents to participate in exploration, games and other behaviors driven by curiosity without explicit reward. In this chapter, we attempt to evaluate intrinsic motivation and transform it into quantized rewards for improving exploration.

\subsection{Reward Shaping}
Reward shaping is a very powerful technique for scaling up RL methods to handle complex problems by supplying additional rewards to the agent to guide its learning process \cite{dorigo1994robot}. The most well-known reward shaping method is potential-based reward shaping (PBRS) \cite{ng1999policy}. We first formally define the potential-based shaping function as follows:
\begin{definition}
	A shaping reward function $f:\mathcal{S}\times\mathcal{A}\times\mathcal{S}\rightarrow\mathbb{R}$ is potential-based if there exists $\phi:\mathcal{S}\rightarrow\mathbb{R}$, such that
	\begin{equation}
		f(\bm{s},\bm{a},\bm{s}')=\gamma\phi(\bm{s}')-\phi(\bm{s}),
	\end{equation}
	for all $\bm{s}\neq\bm{s}_{0},\bm{s}'$.
\end{definition}
For instance, we can utilize the state-value function as the shaping function:
\begin{equation}
	f(\bm{s},\bm{a},\bm{s}')=\gamma V^{\pi}(\bm{s}')-V^{\pi}(\bm{s}).
\end{equation}
Furthermore, we have the following theorem:
\begin{theorem}
	If $f$ is a potential-based shaping function, then every optimal policy in $\mathcal{M}'=\langle  \mathcal{S},\mathcal{A},\mathcal{T},r+f,\rho(\bm{s}_{0}),\gamma\rangle$ will also be an optimal policy in $\mathcal{M}=\langle  \mathcal{S},\mathcal{A},\mathcal{T},r,\rho(\bm{s}_{0}),\gamma\rangle$ (and vice versa).
\end{theorem}
This theorem demonstrates that PBRS can guarantee the policy invariance while adjust the learning process. Consider the following GridWorld game, in which the agent needs to move from the black node to the red node with minimum steps, and the reward is $-1$ per step.

\begin{figure}[h]
	\centering
	\includegraphics[width=0.5\linewidth]{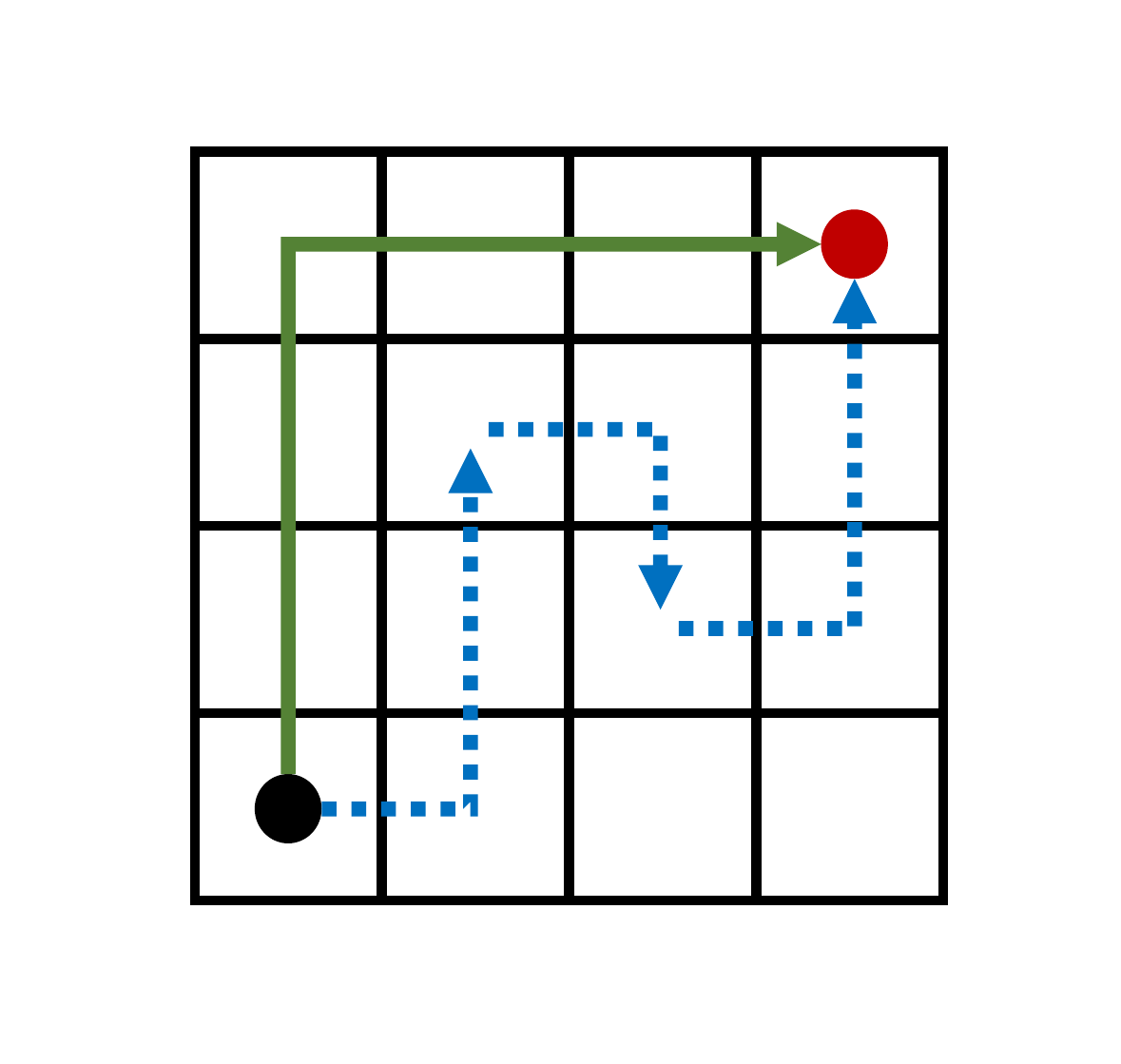}
	\caption{A GridWorld game.}
	\label{fig:gridworld}
\end{figure}

\begin{figure}[h]
	\centering
	\includegraphics[width=0.6\linewidth]{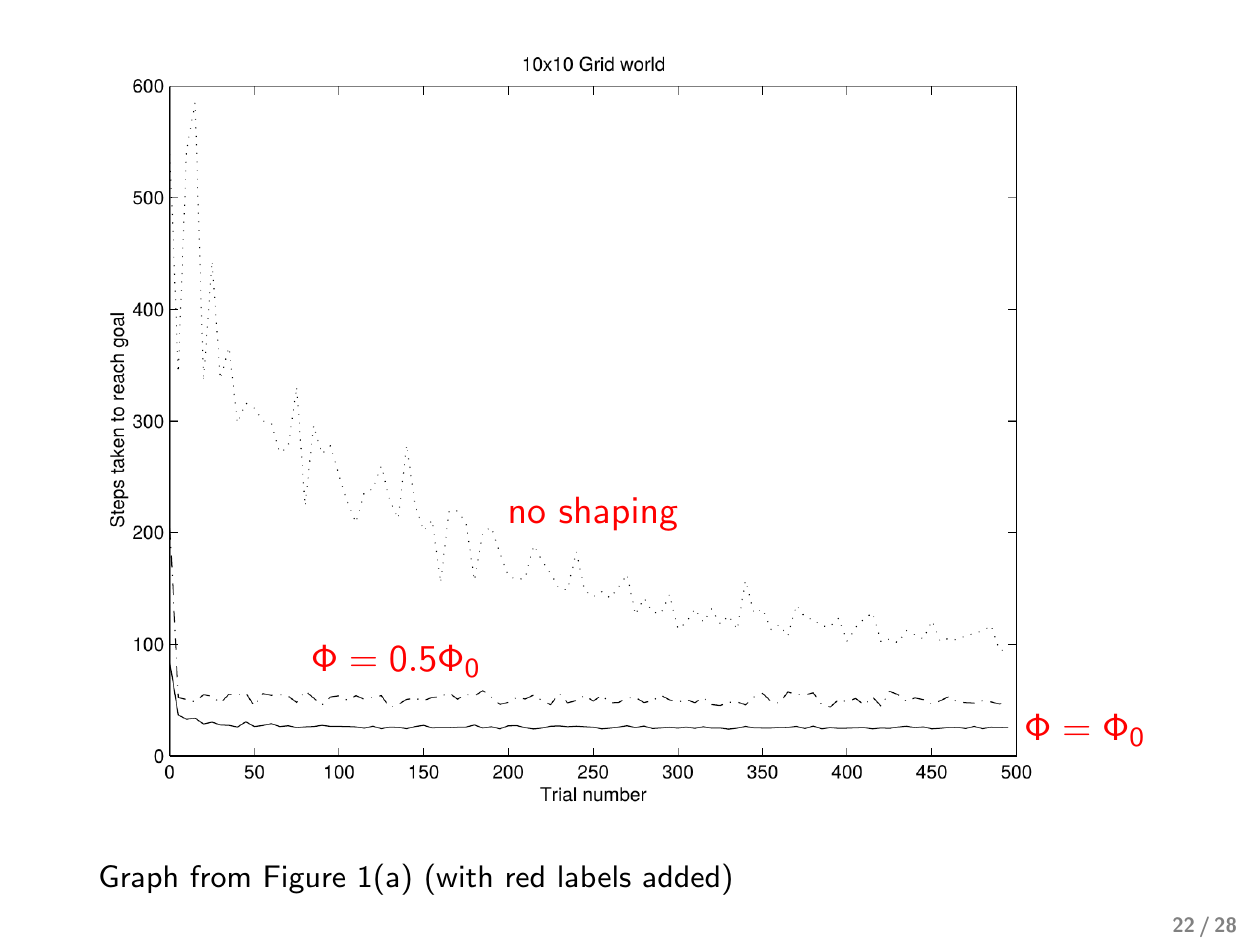}
	\caption{Experiments for $10\times10$ GridWorld.}
	\label{fig:gridworld return}
\end{figure}

Therefore, one estimate of the state-value function is
\begin{equation}
	\phi_{0}(\bm{s})=-D_{\rm M}(\bm{s},\mathrm{GOAL}),
\end{equation}
where $D_{\rm M}$ is the Manhattan distance. Then we use $\phi_{0}$ and $0.5\phi_{0}$ as potential functions, and the performance comparison is illustrated in Figure \ref{fig:gridworld return}. It is obvious that the training process is significantly accelerated with the shaped rewards. 

In the following sections, we will introduce multifarious intrinsic rewards. In general, all of these methods belong to reward shaping approaches. But they may not follow the form of the potential-based functions, {\em i.e.}, they may change what the optimal policy is. However, extensive experiments demonstrate that the intrinsic rewards effectively improve the exploration efficiency and allow the agent to learn better policies \cite{burda2018large}. In particular, most formulations of intrinsic reward can be broadly categorized into two classes:
\begin{itemize}
	\item Encourage the agent to explore novel states or state-action pairs;
	\item Encourage the agent to take actions that reduce the uncertainty of predicting the consequence of its own actions, {\em i.e.}, the knowledge of the environment.
\end{itemize}

\subsection{Novelty-Based Intrinsic Rewards}
\subsubsection{Count-Based Exploration}
We begin with a count-based exploration method proposed by \cite{ostrovski2017count}, which leverages the pseudo-count method to evaluate the novelty of states. Let $\psi$ denotes a density model on a finite space $\mathcal{S}$, and $\psi_{n}(\bm{s})$ is the probability after being trained on a sequence of states $\bm{s}_{1},\dots,\bm{s}_{n}$. Assume $\psi_{n}(\bm{s})>0$ for all $\bm{s},n$, let $\psi'_{n}(\bm{s})$ be the recording probability the model would assign to $\bm{s}$ if it was trained on the same $\bm{s}$ one more time. We say that $\psi$ is learning-positive if $\psi'_{n}(\bm{s})>\psi_{n}(\bm{s})$, and the prediction gain (PG) of $\psi$ is
\begin{equation}
	\mathrm{PG}_{n}(\bm{s})=\log \psi_{n}^{\prime}(\bm{s})-\log \psi_{n}(\bm{s}).
\end{equation}
Then we define the pseudo-count as 
\begin{equation}
	\hat{\mathrm{N}}_{n}(\bm{s})=\frac{\psi_{n}(\bm{s})\left(1-\psi_{n}^{\prime}(\bm{s})\right)}{\psi_{n}^{\prime}(\bm{s})-\psi_{n}(\bm{s})},
\end{equation}
derived by assuming that a single observation of $\bm{s}\in\mathcal{S}$ leads to a unit increase in pseudo-count:
\begin{equation}
	\psi_{n}(\bm{s})=\frac{\hat{\mathrm{N}}_{n}(\bm{s})}{\hat{n}}, \quad \psi_{n}^{\prime}(\bm{s})=\frac{\hat{\mathrm{N}}_{n}(\bm{s})+1}{\hat{n}+1},
\end{equation}
where $\hat{n}$ is the pseudo-count total. Upon certain assumptions on $\psi_{n}$, pseudo-counts grow approximately linearly with real counts. Moreover, the pseudo-count can be approximated as
\begin{equation}
	\hat{\mathrm{N}}_{n}(\bm{s}) \approx\left(e^{\mathrm{PG}_{n}(\bm{s})}-1\right)^{-1}.
\end{equation}
Finally, we define the intrinsic reward at step $n$ as
\begin{equation}
	\hat{r}(\bm{s})=\sqrt{1/\hat{N}_{n}(\bm{s})},
\end{equation}
which encourages the agent to try to re-experience surprising states. On the contrary, if a state is visited multiple times, the agent will visit it with decay frequency.

\subsubsection{Random Network Distillation}
Count-based exploration is a straightforward method for evaluating the state novelty. But it suffers from complicated computation procedures and the pseudo-count method may produce large variance. To address the problem, \cite{burda2018exploration} proposed a random network distillation (RND) method that utilizes neural network to record the state novelty. As shown in Figure \ref{fig:rnd}, RND consists of two neural networks, namely a \textit{target} network and a \textit{predictor} network. The target network is fixed and randomly initialized, while the predictor network is trained on the observation data collected by the agent.

\begin{figure}[h]
	\centering
	\includegraphics[width=0.75\linewidth]{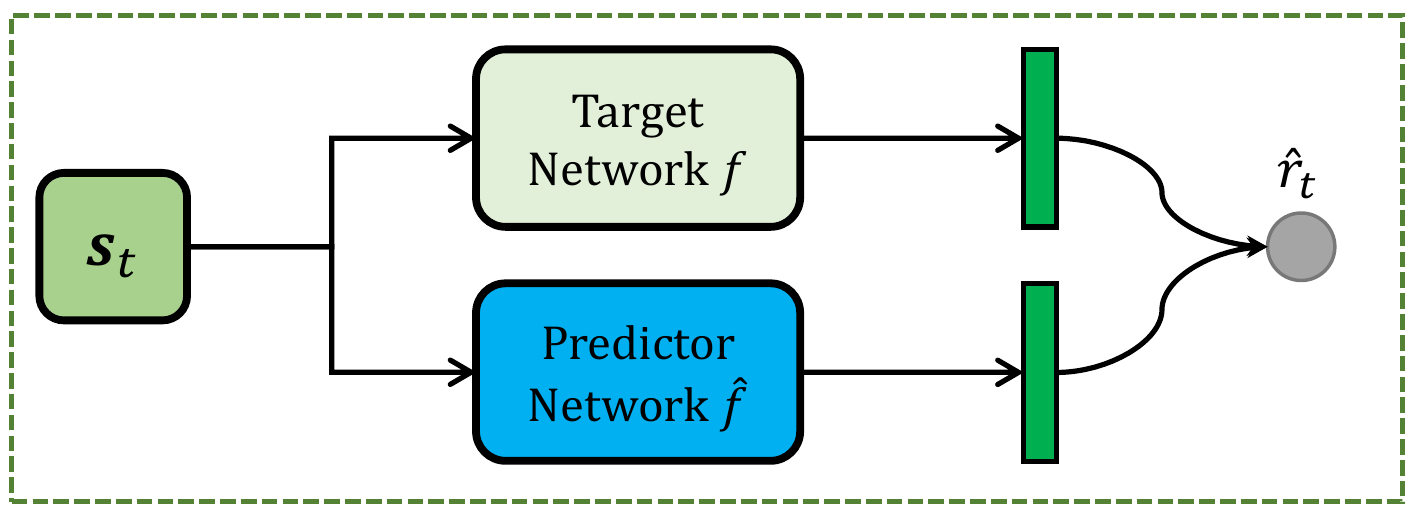}
	\caption{The overview of RND.}
	\label{fig:rnd}
\end{figure}

Let $f:\mathcal{S}\rightarrow\mathbb{R}^{m}$ denote the target network, and $\hat{f}:\mathcal{S}\rightarrow\mathbb{R}^{m}$ denote the predictor network with parameters $\bm{\theta}$, where $m$ is the embedding size. Then the predictor network is trained by minimizing the following loss function:
\begin{equation}
	L(\bm{\theta})=\Vert \hat{f}(\bm{s},\bm{\theta})-f(\bm{s}) \Vert_{2}^{2}.
\end{equation}
Thus the predictor network will be distilled into the target network, and novel states will produce higher prediction error. Finally, we summarize the RND algorithm in Algorithm \ref{algo:rnd}.
\begin{algorithm}
	\caption{Random Network Distillation}
	\label{algo:rnd}
	\begin{algorithmic}
		\STATE \textbf{Input}: A target network $f$, a predictor network $\hat{f}$, a policy $\pi$, a number $N_{\rm opt}$ of optimization steps, a replay buffer $\mathcal{B}$.
		\STATE \textbf{Loop} for each episode:
		\STATE $\quad$ Sample state $\bm{s}_{0}$ from $\rho(\bm{s}_{0})$;
		\STATE $\quad$ Set $t=0$;
		\STATE $\quad$ \textbf{Loop} for each step of episode:
		\STATE $\quad\quad$ Sample $\bm{a}_{t}\sim\pi(\cdot|\bm{s}_{t})$;
		\STATE $\quad\quad$ Sample $\bm{s}_{t+1}\sim\mathcal{T}(\bm{s}_{t+1}|\bm{s},\bm{a})$;
		\STATE $\quad\quad$ Calculate the intrinsic reward $\hat{r}_{t}=\Vert\hat{f}(\bm{s}_{t+1})-f(\bm{s}_{t+1})\Vert$;
		\STATE $\quad\quad$ Save the transition $(\bm{s}_{t},\bm{a}_{t},\bm{s}_{t+1},r_{t},\hat{r}_{t})$ into $\mathcal{B}$.
		\STATE $\quad$ Until $\bm{s}_{t+1}$ is terminal;
		\STATE $\quad$ \textbf{for} $i=1,2,\dots,N_{\rm opt}$ \textbf{do}
		\STATE $\quad\quad$ Sample transitions from $\mathcal{B}$;
		\STATE $\quad\quad$ Optimize the policy using PPO method;
		\STATE $\quad\quad$ Optimize the predictor network;
		\STATE $\quad$ \textbf{end for}
		\STATE \textbf{Output}: The trained policy $\pi$.
	\end{algorithmic}
\end{algorithm}

\begin{figure}[h]
	\centering
	\includegraphics[width=\linewidth]{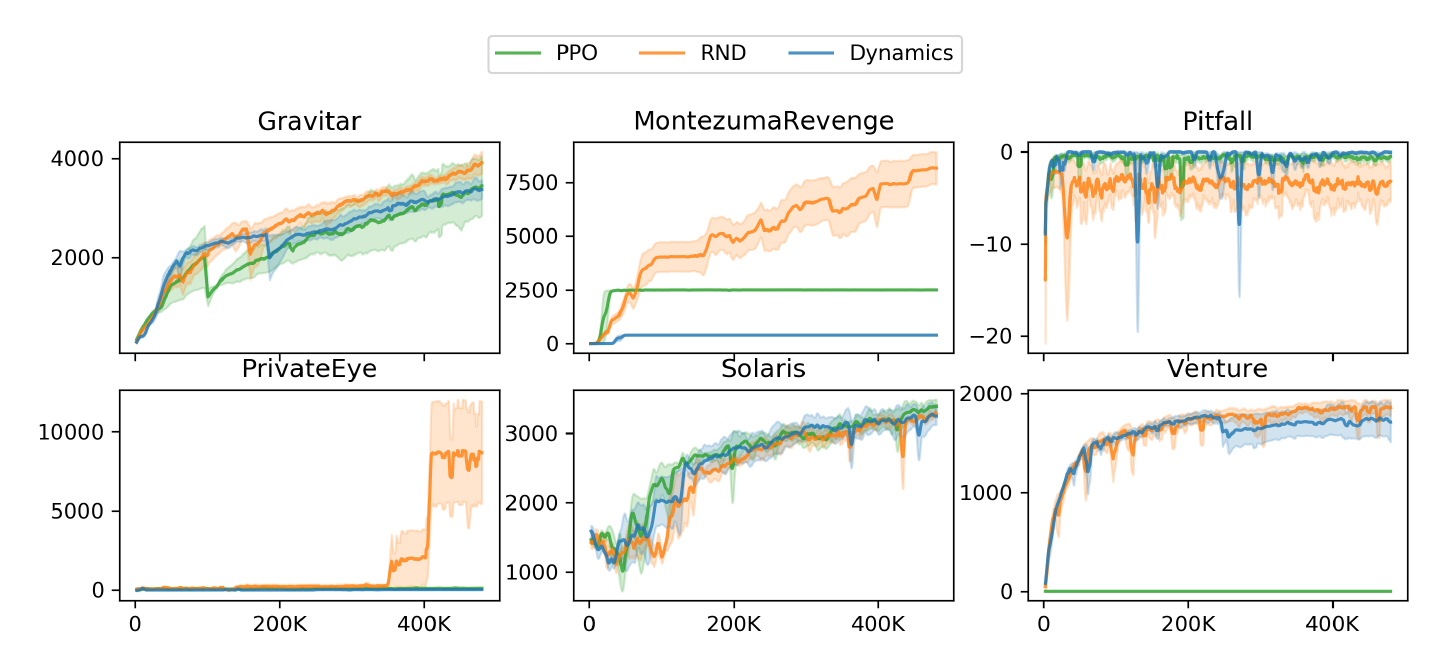}
	\caption{Mean episodic return of RND, dynamics-based exploration method, and PPO.}
	\label{fig:rnd return}
\end{figure}

We next test RND method on six hard-exploration Atari games, \ie, the environments have very sparse and even deceptive rewards. The performance comparison is illustrated in Figure \ref{fig:rnd return}. RND achieves the best performance on Gravitar, Montezuma's Revenge, and Venture, significantly outperforming PPO on the latter two. 

\subsection{Prediction Error-Based Intrinsic Rewards}
\subsubsection{Exploration with Deep Predictive Models}
The second class of exploration methods often learn the dynamic model of the environment, and utilize the prediction error as the intrinsic reward. For instance, \cite{stadie2015incentivizing} proposed a scalable and efficient method for assigning exploration bonuses in large RL problems with complex observations. Let $\sigma(\bm{s})$ denote the encoding of the state $\bm{s}$ and $f:\sigma(\mathcal{S})\times\mathcal{A}\rightarrow\sigma(\mathcal{S})$ be a predictor. $f$ takes an encoded state $\bm{s}$ at time $t$ and the action to predict the encoded successor state $\bm{s}$. For each transition $(\bm{s}_{t},\bm{a}_{t},\bm{s}_{t+1})$, we have the following prediction error:
\begin{equation}
	e(\bm{s}_{t},\bm{a}_{t})=\Vert \sigma(\bm{s}_{t+1})-f(\sigma(\bm{s}_{t}),\bm{a}_{t}) \Vert.
\end{equation}
Then the intrinsic reward is defined as 
\begin{equation}
	\hat{r}(\bm{s}_{t},\bm{a}_{t})=\frac{e(\bm{s}_{t},\bm{a}_{t})}{t*c},
\end{equation}
where $c>0$ is a decaying constant. If we can accurately model the dynamics of a particular state-action pair, it means we have fully comprehended the state and its novelty is lower. On the contrary, if large prediction error is produced, it means more knowledge about that particular area is needed and hence a higher exploration bonus will be assigned. 

\subsubsection{Intrinsic Curiosity Module}
Performing predictions in the raw sensory space like pixels may produce poor behavior and assign same exploration bonuses to all the states. In \cite{stadie2015incentivizing}, the encoding function of state space is learned via an auto-encoder. In contrast, \cite{pathak2017curiosity} proposed an intrinsic curiosity module (ICM) that learns the state feature space using self-supervision method. ICM suggests that an appropriate state feature space needs to exclude the factors that cannot affect the behavior of the agent. Thus it is unnecessary for the agent to learn them. To that end, ICM designs an inverse dynamic model to predict the action of the agent based on its current and next states. The learned feature space only focuses on the changes related to the actions of the agent and ignores the rest. Then this feature space is used to train a forward dynamics model that predicts the feature representation of the next state. The overview of ICM is illustrated in Figure \ref{fig:icm}.

\begin{figure}[h]
	\centering
	\includegraphics[width=0.8\linewidth]{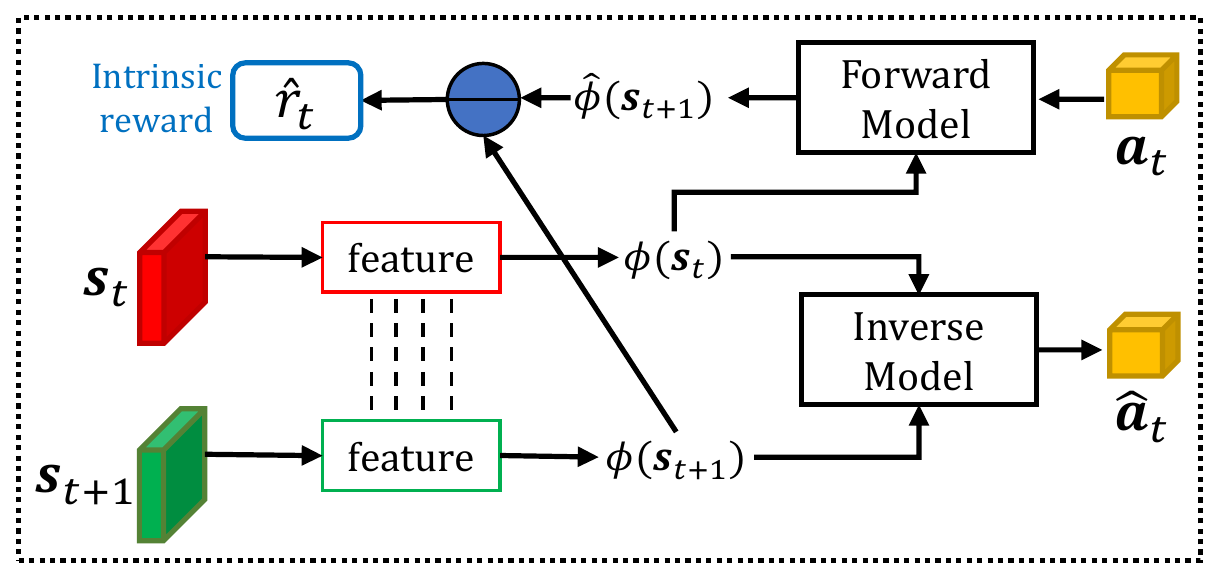}
	\caption{The overview of ICM, where $\ominus$ denotes the Euclidean distance.}
	\label{fig:icm}
\end{figure}

Denote by $I$ and $F$ the inverse model and forward model, and let $\phi$ be the embedding network. The inverse model is optimized by minimizing the following loss function:
\begin{equation}
	L_{I}=D_{I}(\hat{\bm{a}}_{t},\bm{a}_{t}),
\end{equation}
where $\hat{\bm{a}}_{t}=I\big(\phi(\bm{s}_{t}),\phi(\bm{s}_{t+1})\big)$, and $D_{I}$ is some measures that evaluate the discrepancy between the predicted and actual actions. For discrete action space, $D_{I}$ is a softmax distribution across all possible actions. For the forward model, it is trained using a mean squared error (MSE) function:
\begin{equation}
	L_{F} = \frac{1}{2}\Vert \hat{\phi}(\bm{s}_{t+1})-\phi(\bm{s}_{t+1}) \Vert_{2}^{2},
\end{equation}
where $\hat{\phi}(\bm{s}_{t+1})=F\big(\phi(\bm{s}_{t}),\bm{a}_{t}\big)$. Finally, the intrinsic reward is defined as
\begin{equation}
	\hat{r}_{t}=\frac{\eta}{2}\Vert \hat{\phi}(\bm{s}_{t+1})-\phi(\bm{s}_{t+1}) \Vert_{2}^{2},
\end{equation}
where $\eta>0$ is a scaling factor. Since there is no incentive for this feature space to encode
any environmental features that are not affected by the actions of agent, it will receive no rewards for reaching inherently unpredictable states. Thus the agent will be robust to the variation in the environment, such as changes in illumination or presence of distractor objects.

\subsubsection{Generative Intrinsic Reward Module}
ICM trains the reward module using supervised learning, which can only provide deterministic rewards and suffers from overfitting. To address this problem, \cite{yu2020intrinsic} proposed a generative intrinsic reward module (GIRM) based on variational auto-encoder (VAE), which empowers the generalization of intrinsic rewards on unseen state-action pairs \cite{kingma2013auto}. As illustrated in Figure \ref{fig:girm}, GIRM consists of two neural network-based components, namely a recognition network ({\em i.e.} the probabilistic \textit{encoder}) and a generative network ({\em i.e.} the probabilistic \textit{decoder}).

\begin{figure}[h]
	\centering
	\includegraphics[width=0.8\linewidth]{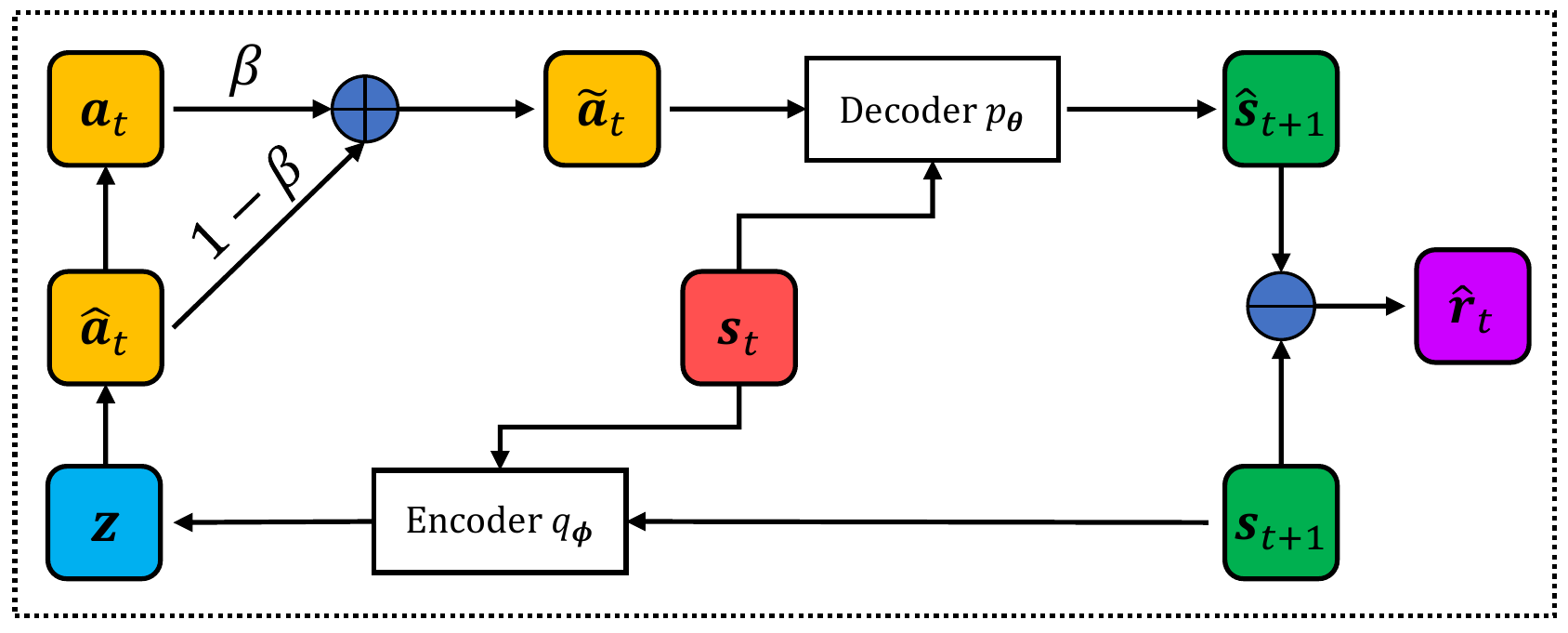}
	\caption{The overview of GIRM, where $\ominus$ denotes the Euclidean distance and $\oplus$ denotes the weighted summation operation.}
	\label{fig:girm}
\end{figure}

The encoder $q_{\bm\phi}(\bm{z}|\bm{s}_{t},\bm{s}_{t+1})$ takes the states $\bm{s}_{t}$ and $\bm{s}_{t+1}$ to perform backward action encoding, and we have
\begin{equation}
	\hat{\bm a}_{t}=\mathrm{Softmax}(\bm{z}).
\end{equation}
We then derive an intermediate action data $\tilde{\bm a}_{t}$ by
\begin{equation}
	\tilde{\bm a}_{t}=\beta\cdot\bm{a}_{t}+(1-\beta)\cdot\hat{\bm a}_{t},
\end{equation}
where $\beta\in(0,1]$ is a weighting coefficient. After that, the decoder $p_{\bm\theta}(\bm{s}_{t+1}|\bm{z},\bm{s}_{t})$ performs forward state transition by taking the action $\tilde{\bm a}_{t}$ and reconstruct the successor state. The GIRM is trained via optimizing the following evidence lower bound:
\begin{equation}
	\begin{aligned}
		L(\bm{s}_{t},\bm{s}_{t+1};\bm{\theta},\bm{\phi})=\mathbb{E}_{q_{\bm\phi}(\bm{z}|\bm{s}_{t},\bm{s}_{t+1})}\left[\log p_{\bm\theta}(\bm{s}_{t+1}|\bm{z},\bm{s}_{t})\right]\\
		+D_{\rm KL}(q_{\bm\phi}(\bm{z}|\bm{s}_{t},\bm{s}_{t+1})\Vert p_{\bm\theta}(\bm{z}|\bm{s}_{t})),
	\end{aligned}
\end{equation}
where $D_{\rm KL}$ is the Kullback–Leibler divergence. Finally, the intrinsic reward is defined as
\begin{equation}
	\hat{r}_{t}=\lambda\Vert\hat{\bm s}_{t+1}-\bm{s}_{t+1}\Vert_{2}^{2},
\end{equation}
where $\lambda$ is a positive scaling weight. In particular, GIRM can generate a family of reward functions by adjusting the mean and variance of the Gaussian distribution of the encoder.

\subsection{Vanishing Intrinsic Rewards}
So far we have introduced two representative classes of intrinsic rewards. However, both of them suffer from the vanishing intrinsic rewards, {\em i.e.}, the intrinsic rewards decrease with visits. The agent will have no additional motivation to explore the environment further once the intrinsic rewards decay to zero. To maintain exploration across episodes, \cite{badia2020never} proposed a never-give-up (NGU) framework that learns mixed intrinsic rewards composed of episodic and life-long state novelty. As illustrated in Figure \ref{fig:ngu}, the episodic novelty module has an episodic memory and an embedding function $f$. At the beginning of each episode, the episodic memory is erased completely. At each time step, it computes an episodic intrinsic reward $r^{\rm episodic}_{t}$ and inserts the embedded state into the memory $\mathcal{E}$. 

\begin{figure}[h]
	\centering
	\includegraphics[width=0.8\linewidth]{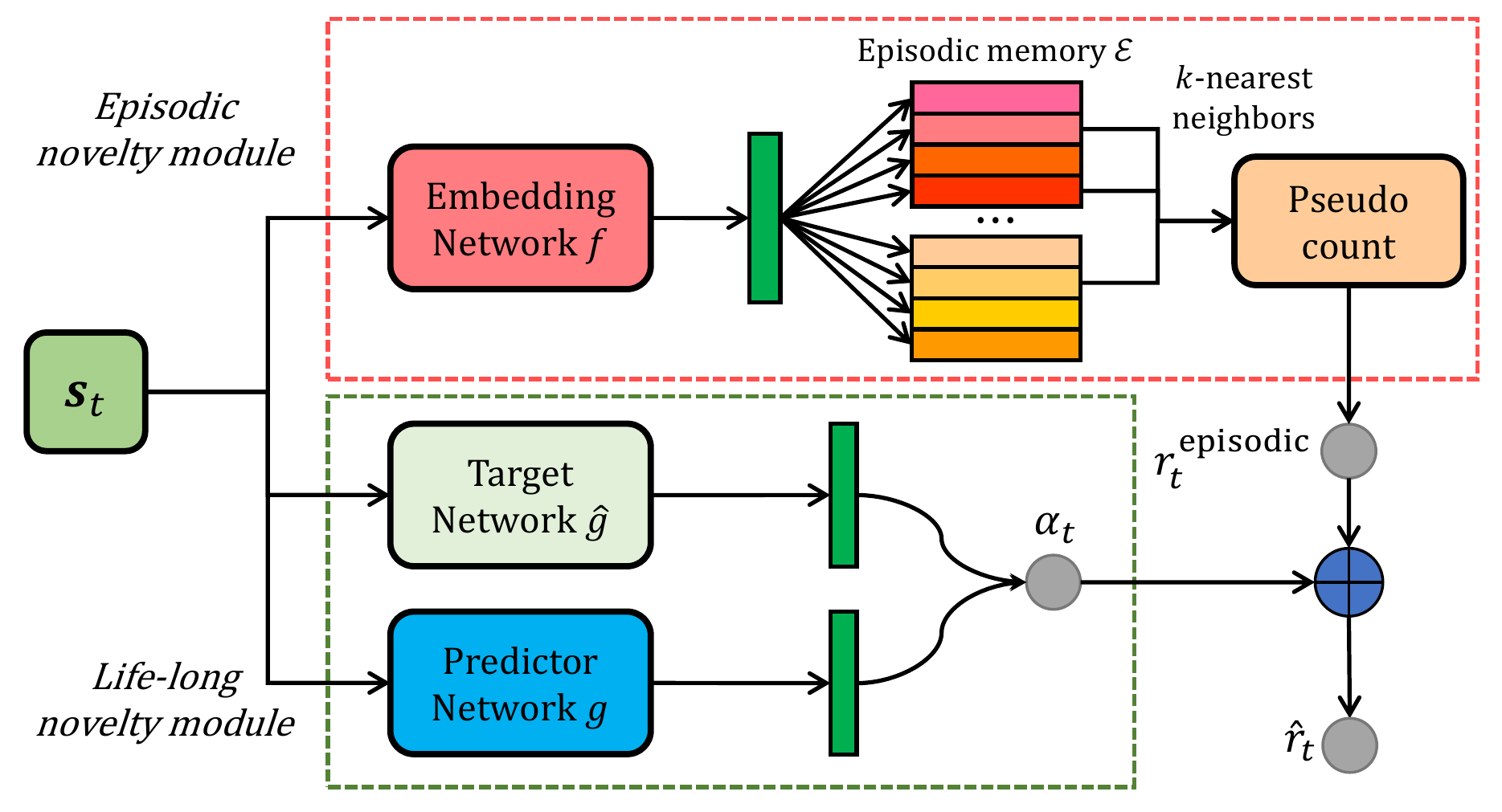}
	\caption{The overview of NGU, where $\oplus$ denotes the weighted summation operation.}
	\label{fig:ngu}
\end{figure}

Consider the episodic memory until time step $t$:
\begin{equation}
	\mathcal{E}=\{\bm{e}_{0},\bm{e}_{1},\dots,\bm{e}_{t}\},
\end{equation}
where $\bm{e}_{t}=f(\bm{s}_{t})$. The episodic intrinsic reward is defined as
\begin{equation}
	r^{\rm episodic}_{t}=\frac{1}{\sqrt{N(\bm{e}_{t})}}\approx\frac{1}{\sqrt{\sum_{\bm{e}_{i}\in N_{k}}K(\bm{e}_{t},\bm{e}_{i})}+c},
\end{equation}
where $N(\bm{e}_{t})$ is the counts of visits of $\bm{e}_{t}$, $N_{k}=\{\bm{e}_{i}\}_{i=1}^{k}$ is the $k$-nearest neighbors of $\bm{e}_{t}$, $K$ is a kernel function, and $c$ guarantees a minimum amount of pseudo-count. The episodic intrinsic rewards encourage the agent to visit as many distinct states as possible within one episode. Since it ignores the novelty interactions across episodes, a state will also be treated as a novel state in the current episode even it has been visited many times.

The life-long novelty module captures the long-term state novelty via RND to control the amount of exploration across episodes. It is realized by multiplicatively modulating the exploration
bonus $r^{\rm episodic}_{t}$ with a life-long curiosity factor $\alpha_{t}$. This factor will vanish over time and reduce the method to using non-shaped rewards. Finally, we define the mixed intrinsic reward as follows:
\begin{equation}
	\hat{r}_{t}=r^{\rm episodic}_{t}\cdot\min\{\max\{\alpha_{t},1\},\beta\},
\end{equation}
where $\beta$ is a maximum reward threshold.

\begin{figure}[h]
	\centering
	\includegraphics[width=1.0\linewidth]{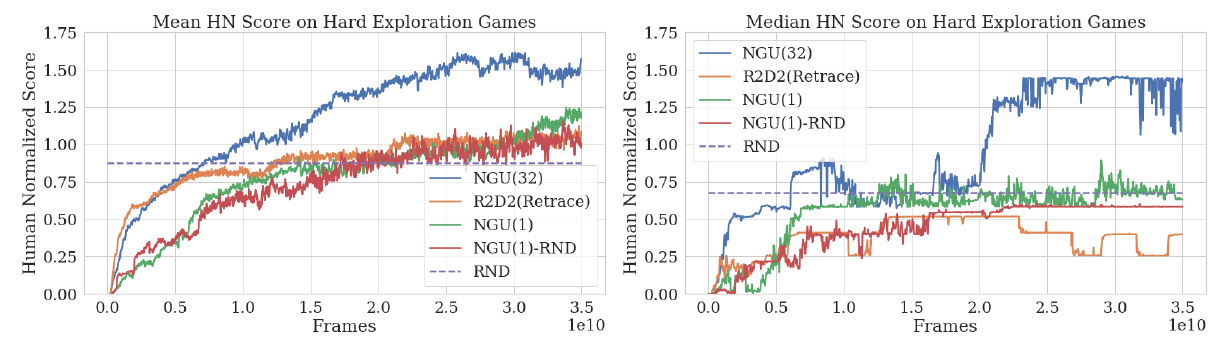}
	\caption{Performance comparison on the $6$ hard-exploration games..}
	\label{fig:ngu return}
\end{figure}

Figure \ref{fig:ngu return} illustrates the performance comparison on six hard-exploration games. NGU achieves on similar or higher average return than the baselines like RND on all hard-exploration tasks. In particular, NGU is the first method that obtains a positive score on \textit{Pitfall!} without using privileged information.

\subsection{Noisy-TV Problem}
Noisy-TV problem is also a critical challenge of the intrinsic reward methods, especially for the prediction error-based approaches \cite{burda2018exploration,savinov2018episodic}. Consider a RL agent walking in a complex maze, and it is rewarded by seeking novel experience. As shown in Figure \ref{fig:tv}, a television with unpredictable random noise will keep producing high prediction error and attract the attention of the agent. Since the agent will consistently receive exploration bonuses from the noisy television, it will make no further progress and just standstill.

\begin{figure}[h]
	\centering
	\includegraphics[width=0.8\linewidth]{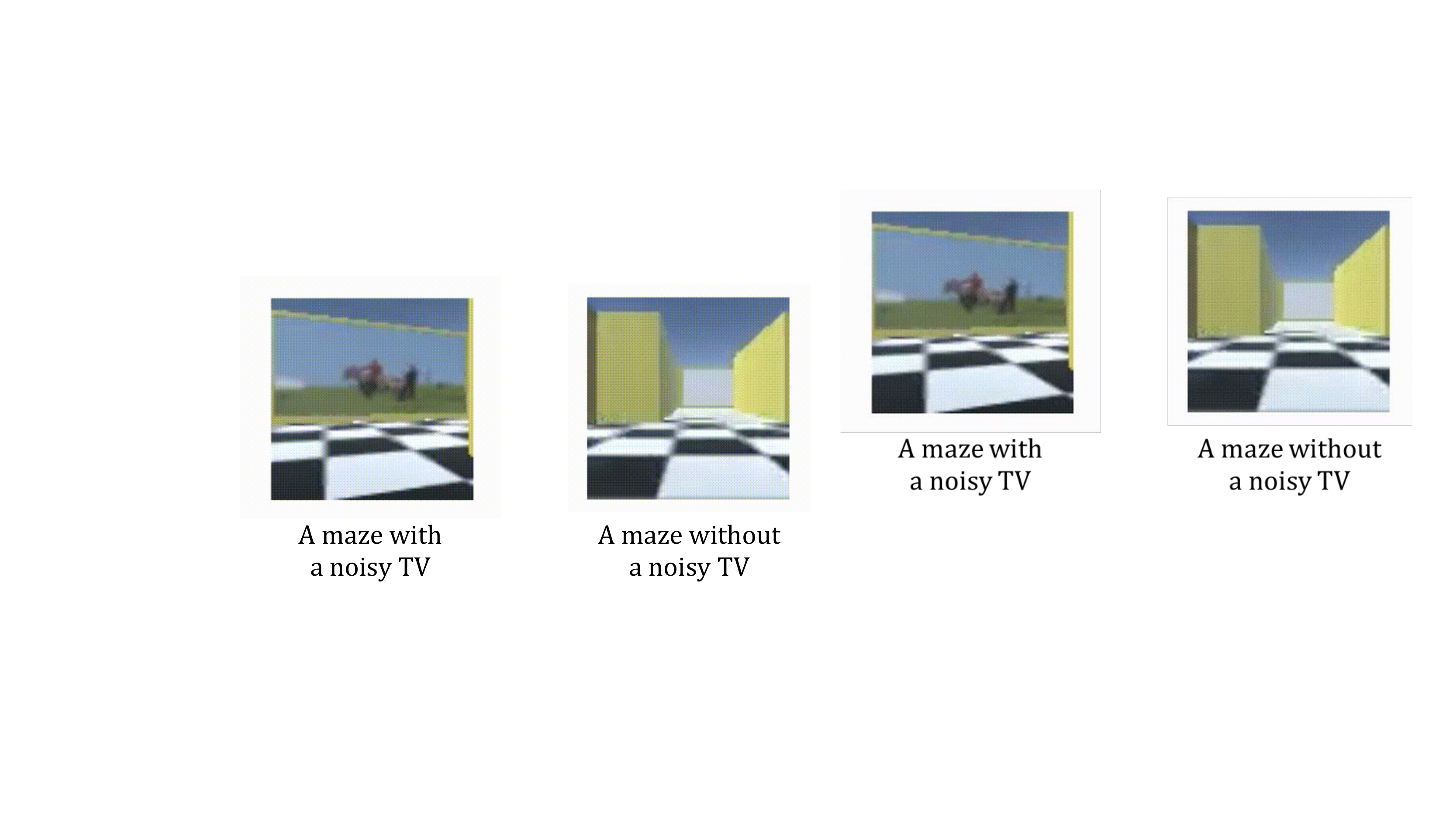}
	\caption{The noisy television will keep playing different pictures.}
	\label{fig:tv}
\end{figure}

To address the problem, \cite{savinov2018episodic} proposed an episodic curiosity module that utilizes reachability between states as novelty bonus. In contrast, \cite{raileanu2020ride} propsoed a more straightforward framework entitled rewarding-impact-driven-exploration (RIDE). In particular, RIDE considers the procedurally-generated environments, in which the environment is constructed differently in every episode. Therefore, learning in such environments requires the agent to perform exploration more actively.

\begin{figure}[h]
	\centering
	\includegraphics[width=0.8\linewidth]{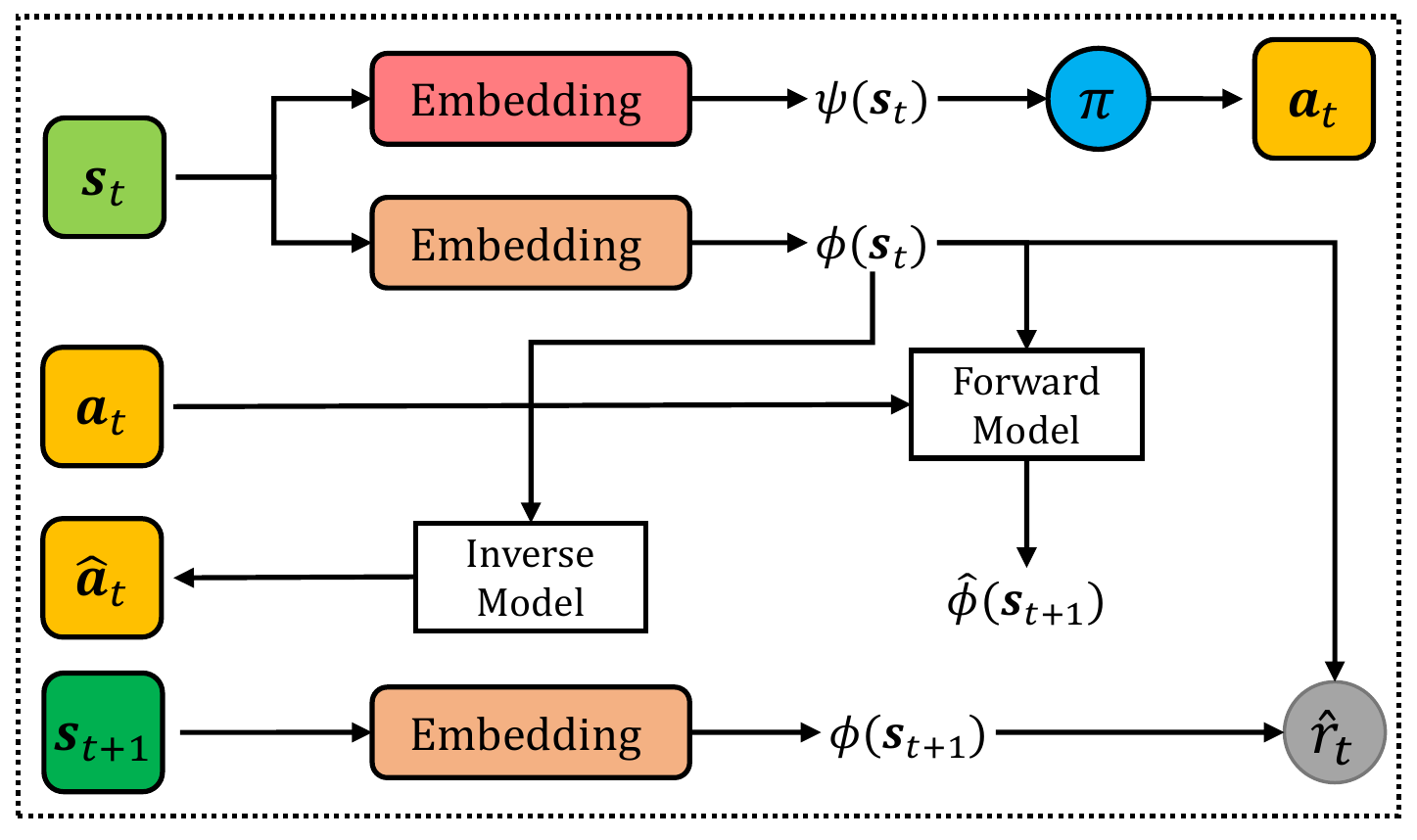}
	\caption{The overview of RIDE.}
	\label{fig:ride}
\end{figure}

As shown in Figure \ref{fig:ride}, RIDE inherits the inverse-forward pattern of ICM, in which two dynamic models are leveraged to reconstruct the transition process. RIDE defines the intrinsic reward as the difference between the consecutive embedded states:
\begin{equation}
	\hat{r}_{t}=\frac{\Vert\phi(\bm{s}_{t+1})-\phi(\bm{s}_{t})\Vert_{2}}{\sqrt{N_{\rm ep}(\bm{s}_{t+1})}},
\end{equation}
where $\phi(\cdot)$ is the embedding network and $N_{\rm ep}$ is the episodic visitation counts. $N_{\rm ep}$ is used to discount the intrinsic reward, which guarantees that the agent will not linger between a sequence of states with a large difference in their embeddings. Through this intrinsic reward, the agent is encouraged to take actions that results in large state change to improve the exploration. In particular, RIDE can also overcome the problem of vanishing intrinsic rewards.

\section{\Ry State Entropy Maximization}
\label{chap:four}

In Chapter \ref{chap:three}, we discussed various intrinsically-motivated methods that utilize the reward shaping to improve the exploration of state space. For instance, NGU leverages the mixed state novelty to encourage the agent to visit as many distinct states as possible. NGU overcomes the problem of the vanishing intrinsic rewards and provides sustainable exploration incentives. However, methods like NGU and RIDE pay excessive attention to specific states while failing to reflect the global exploration completeness. Furthermore, they suffer from high computational complexity and performance loss incurred by abundant auxiliary models. In this chapter, we will first propose to reformulate the exploration problem before establishing a more efficient exploration method.

\subsection{Coupon Collector's Problem}
In probability theory, the coupon collector's problem (CCP) is described as the contest of collecting all coupons and winning \cite{flajolet1992birthday}. Given $n$ coupons, how many coupons do you need to draw with replacement before having drawn each coupon at least once? For instance, it takes about $172$ trials on average to collect about all coupons when $n=40$. To guarantee the completeness of exploration, the agent is expected to visit all possible states during training. Such an objective can be also considered as a CCP conditioned upon a nonuniform probability distribution, in which the agent is the collector and the states are the coupons. Denote by $d^{\pi}(\bm{s})$ the state distribution induced by the policy $\pi$. Assuming that the agent takes $\tilde{T}$ environment steps to finish the collection, we can compute the expectation of $\tilde{T}$ as
\begin{equation}\label{eq:Epi}
	\mathbb{E}_{\pi}(\tilde{T})=\int_{0}^{\infty}\bigg( 1-\prod_{i=1}^{|\mathcal{S}|}(1-e^{-d^{\pi}(\bm{s}_{i})t}) \bigg)\,dt,
\end{equation}
where $\left|\cdot\right|$ stands for the cardinality of the enclosed set $\mathcal{S}$.

For simplicity of notation, we omit the superscript in $d^{\pi}(\bm{s})$ in the sequel. Efficient exploration aims to find a policy that optimizes $\min_{\pi\in\Pi}\:\mathbb{E}_{\pi}(\tilde{T})$. However, it is non-trivial to evaluate Eq.~\eqref{eq:Epi} due to the improper integral, not to mention solving the optimization problem. To address the problem, it is common to leverage the Shannon entropy to make a tractable objective function, which is defined as
\begin{equation}\label{eq:Hd}
	H(d)=-\mathbb{E}_{\bm{s}\sim d(\bm{s})}\big[\log d({\bm s})\big].
\end{equation}

\subsection{Random Encoders for Efficient Exploration}
To estimate the state entropy, \cite{seo2021state} introduced a $k$-nearest neighbor entropy estimator \cite{singh2003nearest} and proposed a random-encoders-for-efficient-exploration (RE3) method. Given a trajectory $\tau=(\bm{s}_{0},\bm{a}_{0},\dots,\bm{a}_{T-1},\bm{s}_{T})$, RE3 first uses a randomly initialized DNN to encode the visited states. Denote by $\{\bm{x}_{i}\}_{i=0}^{T-1}$ the encoding vectors of observations, RE3 estimates the entropy of state distribution $d(\bm{s})$ using a $k$-nearest neighbor entropy estimator \cite{singh2003nearest}:
\begin{equation}\label{eq:knn ee}
	\begin{aligned}
		\hat{H}_{T}^{k}(d) &= \frac{1}{T}\sum_{i=0}^{T-1}\log \frac{T\cdot \Vert \bm{x}_{i} - \tilde{\bm{x}}_{i} \Vert_{2}^{m} \cdot \pi}{k\cdot\Gamma(\frac{m}{2}+1)}+\log k-\Psi(k) \\
		&\propto \frac{1}{T}\sum_{i=0}^{T-1}\log\Vert \bm{x}_{i} - \tilde{\bm{x}}_{i} \Vert_{2},
	\end{aligned}
\end{equation}
where $\tilde{\bm{x}}_{i}$ is the $k$-nearest neighbor of $\bm{x}_{i}$ within the set $\{\bm{x}_{i}\}_{i=0}^{T-1}$, $m$ is the dimension of the encoding vectors, and $\Gamma(\cdot)$ is the Gamma function, and $\Psi(\cdot)$ is the digamma function. Note that $\pi$ in Eq.~\eqref{eq:knn ee} denotes the ratio between the circumference of a circle to its diameter. Equipped with Eq.~\eqref{eq:knn ee}, the total reward for each transition $(\bm{s}_{t},\bm{a}_{t},r_{t},\bm{s}_{t+1})$ is computed as:
\begin{equation}
	r^{\rm total}=r(\bm{s}_{t},\bm{a}_{t})+\lambda_{t} \cdot \log(\Vert \bm{x}_{t} - \tilde{\bm{x}}_{t} \Vert_{2} + 1),
\end{equation}
where $\lambda_{t}=\lambda_{0}(1-\kappa)^{t},\lambda_{t}\geq 0$ is a weight coefficient that decays over time, $\kappa$ is a decay rate.

\begin{figure}
	\centering
	\includegraphics[width=\linewidth]{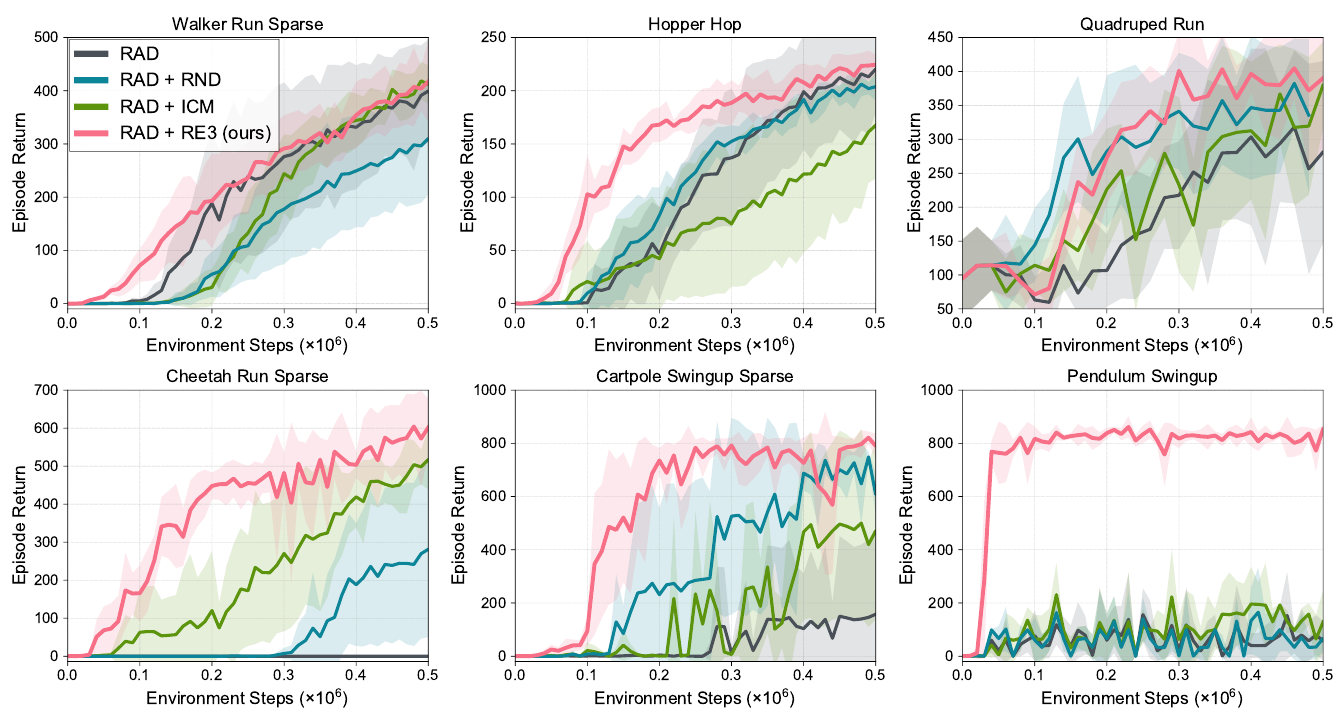}
	\caption{Performance on locomotion tasks from DeepMind control suite.}
	\label{fig:re3 return}
\end{figure}

In sharp contrast to the methods of Chapter \ref{chap:three}, RE3 does not need abundant auxiliary models or storages to record and analyze the learning procedure. Thus it is very efficient and can be easily employed in arbitrary tasks. We next tested the RE3 on locomotion tasks from DeepMind control suite. As illustrated in Figure \ref{fig:re3 return}, RE3 significantly outperforms other exploration methods in terms of sample-efficiency. Despite the advantages of RE3, the Shannon entropy-based objective may lead to a policy that visits some states with a vanishing probability. In the following
section, we will employ a representative example to demonstrate the practical drawbacks of Eq.~\eqref{eq:Hd} and introduce the \Ry entropy to address the problem.

\subsection{\Ry State Entropy Maximization}
\subsubsection{\Ry State Entropy}
We first formally define the \Ry entropy as follows:
\begin{definition}[\Ry Entropy]\label{def:re}
	Let $X\in\mathbb{R}^{m}$ be a random vector that has a density function $f(\bm{x})$ with respect to Lebesgue measure on $\mathbb{R}^{m}$, and let $\mathcal{X}=\{\bm{x}\in\mathbb{R}^{m}:f(\bm{x})>0\}$ be the support of the distribution. The \Ry entropy of order $\alpha\in(0,1)\cup(1,+\infty)$ is defined as \cite{zhang2021exploration}:
	\begin{equation}\label{eq:hf}
		H_{\alpha}(f)=\frac{1}{1-\alpha}\log \int_{\mathcal{X}}f^{\alpha}(\bm{x})d \bm{x}.
	\end{equation}
\end{definition}

Using Definition~\ref{def:re}, we propose the following \Ry state entropy (RISE):
\begin{equation}
	H_{\alpha}(d)=\frac{1}{1-\alpha}\log \int_{\mathcal{S}}d^{\alpha}(\bm{s})d\bm{s}.
\end{equation}

\begin{figure*}[h]
	\centering
	\includegraphics[width=\linewidth]{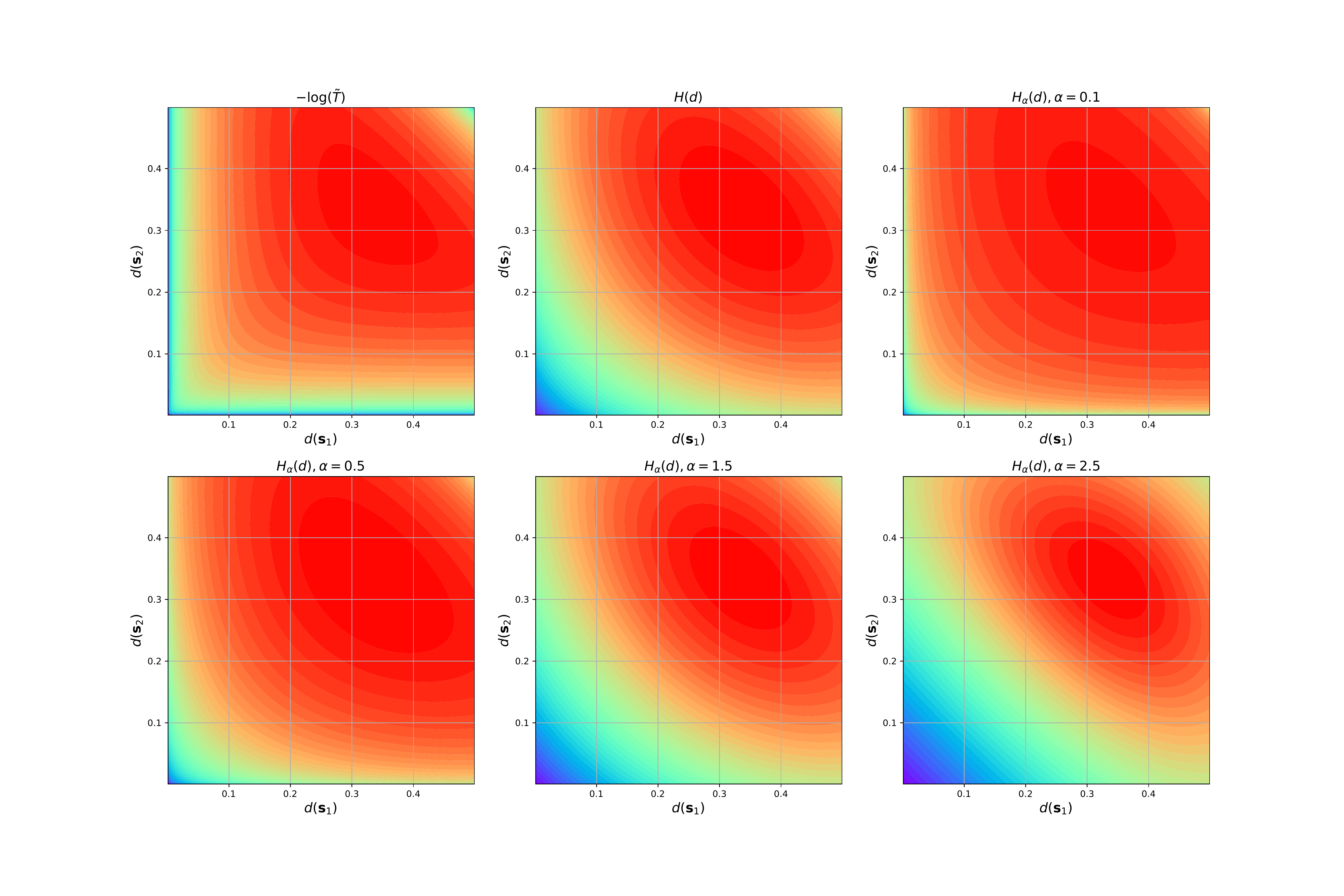}
	\caption{The contours of different objective functions when $|\mathcal{S}|=3$.}
	\label{fig:contour}
\end{figure*}

Fig.~\ref{fig:contour} use a toy example to visualize the contours of different objective functions when an agent learns from an environment characterized by only three states. As shown in Fig.~\ref{fig:contour}, $-\log(\tilde{T})$ decreases rapidly when any state probability approaches zero, which prevents the agent from visiting a state with a vanishing probability while encouraging the agent to explore the infrequently-seen states. In contrast, the Shannon entropy remains relatively large as the state probability approaches zero. Interestingly, Fig.~\ref{fig:contour} shows that this problem can be alleviated by the \Ry entropy as it better matches $-\log(\tilde{T})$. The Shannon entropy is far less aggressive in penalizing small probabilities, while the \Ry entropy provides more flexible exploration intensity. 

\subsubsection{Theoretical Analysis}
To maximize $H_{\alpha}(d)$, we consider using a maximum entropy policy computation (MEPC) algorithm proposed by \cite{hazan2019provably}, which uses the following two oracles: 
\begin{definition}[Approximating planning oracle]
	Given a reward function $r:\mathcal{S}\rightarrow\mathbb{R}$ and a gap $\epsilon$, the planning oracle returns a policy by $\pi=O_{\rm AP}(r,\epsilon_1)$, such that
	\begin{equation}
		V^{\pi}\geq\max_{\pi\in\Pi}V^{\pi}-\epsilon,
	\end{equation}
	where $V^{\pi}$ is the state-value function.
\end{definition}
\begin{definition}[State distribution estimation oracle]
	Given a gap $\epsilon$ and a policy $\pi$, this oracle estimates the state distribution by $\hat{d}=O_{\rm DE}(\pi,\epsilon)$, such that
	\begin{equation}
		\Vert d-\hat{d}\Vert_{\infty}\leq\epsilon.
	\end{equation}
\end{definition}
Given a set of stationary policies $\hat{\Pi}=\{\pi_0,\pi_1,\dots\}$, we define a mixed policy as $\pi_{\rm mix}=(\bm{\omega},\hat{\Pi})$, where $\omega$ contains the weighting coefficients. Then the induced state distribution is
\begin{equation}
	d^{\pi_{\rm mix}}=\sum_{i}\bm{\omega}_{i}d^{\pi_i}(\bm{s}).
\end{equation}
Finally, the workflow of MEPC is summarized in Algorithm \ref{algo:mpec}.

\begin{algorithm}[h]
	\caption{MEPC}
	\label{algo:mpec}
	\begin{algorithmic}[1]
		\STATE Set the number of iterations $T$, step size $\eta$, planning oracle error tolerance $\epsilon_1>0$, and state distribution oracle error tolerance $\epsilon_2>0$;
		\STATE Initialize $\hat{\Pi}_0=\{\pi_0\}$, where $\pi_0$ is an arbitrary policy;
		\STATE Initialize $\omega_0=1$;
		\FOR {$t=0,\dots,T-1$}
		\STATE Invoke the state distribution oracle on $\pi_{{\rm mix}, t}=(\omega,\hat{\Pi}_t)$:
		\begin{equation}\nonumber
			\hat{d}^{\pi_{{\rm mix}, t}}=O_{\rm AP}(r,\epsilon_1);
		\end{equation}
		
		\STATE Define the reward function $r_{t}$ as
		\begin{equation}\nonumber
			r_{t}(\bm{s})=\nabla H_{\alpha}(\hat{d}^{\pi_{{\rm mix}, t}});
		\end{equation}
		
		\STATE Approximate the optimal policy on $r_t$:
		\begin{equation}\nonumber
			\pi_{t+1}=O_{\rm DE}(\pi,\epsilon_2);
		\end{equation}
		
		\STATE Update $\pi_{{\rm mix}, t}=(\omega_{t+1},\hat{\Pi})$
		\begin{equation}\nonumber
			\begin{aligned}
				\hat{\Pi}_{t+1}&=(\pi_0,\dots,\pi_t,\pi_{t+1})\\
				\omega_{t+1}&=\left((1-\eta)\omega_{t},\eta\right);
			\end{aligned}
		\end{equation}
		
		\ENDFOR
		\STATE Return $\pi_{{\rm mix}, T}=(\omega_T,\hat{\Pi}_T)$.
	\end{algorithmic}
\end{algorithm}

Consider the discrete case of \Ry state entropy and set $\alpha\in(0,1)$, we have 
\begin{equation}
	H_{\alpha}(d)=\frac{1}{1-\alpha}\log\sum_{\bm{s}\in\mathcal{S}}d^{\alpha}(\bm {s}).
\end{equation}
Since the logarithmic functions are monotonically increasing functions, the maximization of $\tilde{H}_{\alpha}(d)$ can be achieved by maximizing the following function:
\begin{equation}
	\tilde{\mathcal{H}}_{\alpha}(d)=\frac{1}{1-\alpha}\sum_{\bm{s}\in\mathcal{S}}d^{\alpha}(\bm {s}).
\end{equation}
As $\tilde{\mathcal{H}}_{\alpha}(d)$ is not smooth, we propose to replace $\tilde{\mathcal{H}}_{\alpha}(d)$ with a smoothed $\tilde{H}_{\alpha,\sigma}(d)$ defined as
\begin{equation}
	\tilde{H}_{\alpha,\sigma}(d)=\frac{1}{1-\alpha}\sum_{\bm{s}\in\mathcal{S}}(d(\bm {s})+\sigma)^{\alpha},
\end{equation}
where $\sigma>0$.
\begin{lemma}\label{lemma:smooth}
	$\tilde{H}_{\alpha,\sigma}(d)$ is $\beta$-smooth, such that
	\begin{equation}
		\Vert\nabla\tilde{H}_{\alpha,\sigma}(d)-\nabla\tilde{H}_{\alpha,\sigma}(d')\Vert_{\infty}\leq \beta\Vert d-d'\Vert_{\infty},
	\end{equation}
	where $\beta=\alpha\sigma^{\alpha-2}$.
\end{lemma}
\begin{proof}
	Since $\nabla^{2}\tilde{H}_{\alpha,\sigma}(d)=-\alpha(d(\bm{s})+\sigma)^{\alpha-2}$ is a diagonal matrix, we have
	\begin{equation}
		\begin{aligned}
			&\Vert\nabla\tilde{H}_{\alpha,\sigma}(d)-\nabla\tilde{H}_{\alpha,\sigma}(d')\Vert_{\infty}\\
			&\leq\max_{\varsigma\in[0,1]}|\nabla^{2}\tilde{H}_{\alpha,\sigma}(\varsigma d+(1-\varsigma)d')|\cdot\Vert d-d'\Vert_{\infty}\\
			&\leq \alpha\sigma^{\alpha-2}\Vert d-d'\Vert_{\infty},
		\end{aligned}
	\end{equation}
	where the first inequality follows the Taylor's theorem. This concludes the proof.
\end{proof}

From Lemma~\ref{lemma:smooth}, the following theorem can be derived.
\begin{theorem}\label{theorem:sample complexity}
	For any $\epsilon>0$ with $\epsilon_1=0.1\epsilon, \epsilon_2=0.1\beta^{-1}\epsilon$ and $\eta=0.1\beta^{-1}\epsilon$, the following inequality holds
	\begin{equation}
		\tilde{H}_{\alpha,\sigma}(d^{\pi_{{\rm mix},T}})\geq \max_{\pi\in\Pi}\tilde{H}_{\alpha,\sigma}(d^{\pi})-\epsilon,
	\end{equation}
	if Algorithm \ref{algo:mpec} is run for $T$ iterations with
	\begin{equation}\label{eq:Tbound}
		T\geq \frac{10\alpha\sigma^{\alpha-2}}{\epsilon}\log\frac{10\alpha\sigma^{\alpha-1}}{(1-\alpha)\epsilon}.
	\end{equation}
\end{theorem}
\begin{proof}
	Equipped with Lemma \ref{lemma:smooth}, let $\pi^{*}=\underset{\pi\in\Pi}{\rm argmax}\:\tilde{H}_{\alpha,\sigma}(d)$, we have (proved by \cite{hazan2019provably}):
	\begin{equation}
		\begin{aligned}
			\tilde{H}_{\alpha,\sigma}(d^{\pi^*})-\tilde{H}_{\alpha,\sigma}(d^{\pi_{{\rm mix},T}})\\
			\leq B\exp(-T\eta)+2\beta\epsilon_2+\epsilon_1+\eta\beta,
		\end{aligned}
	\end{equation}
	where $\Vert\nabla\tilde{H}_{\alpha,\sigma}(d)\Vert_{\infty}\leq B=\frac{\alpha}{1-\alpha}\sigma^{\alpha-1}$. Thus it suffices to set $\epsilon_1=0.1\epsilon,\epsilon_2=0.1\beta^{-1}\epsilon,\eta=0.1\beta^{-1}\epsilon,T=\eta^{-1}\log10B\epsilon^{-1}$. When Algorithm \ref{algo:mpec} is run for 
	\begin{equation}
		T\geq 10\beta\epsilon^{-1}\log10B\epsilon^{-1},
	\end{equation}
	it holds
	\begin{equation}
		\tilde{H}_{\alpha,\sigma}(d^{\pi_{{\rm mix},T}})\geq \max_{\pi\in\Pi}\tilde{H}_{\alpha,\sigma}(d^{\pi})-\epsilon.
	\end{equation}
	This concludes the proof.
\end{proof}
Theorem \ref{theorem:sample complexity} demonstrates that the proposed \Ry state entropy can be practically maximized with computational complexity given in Theorem~\ref{theorem:sample complexity} by exploiting MEPC.
Inspection of Eq.~\eqref{eq:Tbound} reveals that the lower bound of $T$ is a function of $\alpha$. In other words, we can expedite the exploration process through reducing the lower bound of $T$ by adjusting the value of $\alpha$. This observation is consistent with the analytical results reported in \cite{zhang2021exploration}.

\subsubsection{Fast Entropy Estimation}

However, it is non-trivial to apply MEPC when handling complex environments with high-dimensional observations. To address the problem, we propose to utilize the following $k$-nearest neighbor estimator to realize efficient estimation of the \Ry entropy \cite{leonenko2008class}. Note that $\pi$ in Eq.~\eqref{eq:estimator} denotes the ratio between the circumference of a circle to its diameter.

\begin{theorem}[Estimator]\label{def:re estimation}
	Denote by  $\{X_{i}\}_{i=1}^{N}$ a set of independent random vectors from the distribution $X$. For $k<N,k\in\mathbb{N}$, $\tilde{X}_{i}$ stands for the $k$-nearest neighbor of $X_{i}$ among the set. We estimate the \Ry entropy using the sample mean as follows:
	\begin{equation}\label{eq:estimator}
		\hat{H}^{k,\alpha}_{N}(f)=\frac{1}{N}\sum_{i=1}^{N}\big[ (N-1)V_{m}C_{k}\Vert X_{i}-\tilde{X}_{i} \Vert^{m} \big]^{1-\alpha},
	\end{equation}
	where $C_{k}=\bigg[ \frac{\Gamma(k)}{\Gamma(k+1-\alpha)} \bigg]^{\frac{1}{1-\alpha}}$ and $V_{m}=\frac{\pi^{\frac{m}{2}}}{\Gamma(\frac{m}{2}+1)}$ is the volume of the unit ball in $\mathbb{R}^{m}$, and $\Gamma(\cdot)$ is the Gamma function, respectively. Moreover, it holds
	\begin{equation}
		\lim_{N\rightarrow\infty}\hat{H}^{k,\alpha}_{N}(f)=H_{\alpha}(f).
	\end{equation}
\end{theorem}
\begin{proof}
	See proof in \cite{leonenko2008class}.
\end{proof}

Given a trajectory $\tau=\{\bm{s}_{0},\bm{a}_{0},\dots,\bm{a}_{T-1},\bm{s}_{T}\}$ collected by the agent, we approximate the \Ry state entropy in Eq.~\eqref{eq:hf} using Eq.~\eqref{eq:estimator} as
\begin{equation}
	\begin{aligned}
		\hat{H}^{k,\alpha}_{T}(d)&=\frac{1}{T}\sum_{i=0}^{T-1}\big[ (T-1) V_{m}C_{k}\Vert \bm{y}_{i}-\tilde{\bm{y}}_{i} \Vert^{m} \big]^{1-\alpha},\\
		&\propto\frac{1}{T}\sum_{i=0}^{T-1} \Vert \bm{y}_{i}-\tilde{\bm{y}}_{i} \Vert^{1-\alpha},
	\end{aligned}
\end{equation}
where $\bm{y}_{i}$ is the encoding vector of $\bm{s}_i$ and $\tilde{\bm{y}}_{i}$ is the $k$-nearest neighbor of $\bm{y}_{i}$. After that, we define the intrinsic reward that takes each transition as a particle:
\begin{equation}\label{eq:intrinsic reward}
	\hat{r}(\bm{s}_{i})=\Vert \bm{y}_{i}-\tilde{\bm{y}}_{i} \Vert^{1-\alpha},
\end{equation}
where $\hat{r}(\cdot)$ is used to distinguish the intrinsic reward from the extrinsic reward $r(\cdot)$. Eq.~\eqref{eq:intrinsic reward} indicates that the agent needs to visit as more distinct states as possible to obtain higher intrinsic rewards.

Such an estimation method requires no additional auxiliary models, which significantly promotes the learning efficiency. Equipped with the intrinsic reward, the total reward of each transition $(\bm{s}_t,\bm{a}_t,\bm{s}_{t+1})$ is computed as
\begin{equation}
	r^{\rm total}_t=r({\bm s}_t,{\bm a}_t)+\lambda_{t}\cdot\hat{r}(\bm{s}_t)+\zeta\cdot H(\pi(\cdot|\bm{s}_t)),
\end{equation}
where $H(\pi(\cdot|\bm{s}_t))$ is the action entropy regularizer for improving the exploration on action space, $\lambda_{t}=\lambda_{0}(1-\kappa)^{t}$ and $\zeta$ are two non-negative weight coefficients, and $\kappa$ is a decay rate. 


\subsection{Robust Representation Learning}
While the \Ry state entropy encourages exploration in high-dimensional observation spaces, several implementation issues have to be addressed in its practical deployment. First of all, observations have to be encoded into low-dimensional vectors in calculating the intrinsic reward. While a randomly initialized neural network can be utilized as the encoder as proposed in \cite{seo2021state}, it cannot handle more complex and dynamic tasks, which inevitably incurs performance loss. Moreover, since it is less computationally expensive to train an encoder than RL, we propose to leverage the VAE to realize efficient and robust embedding operation, which is a powerful generative model based on the Bayesian inference \cite{kingma2013auto}. As shown in Fig.~\ref{fig:rise_a}, a standard VAE is composed of a recognition model and a generative model. These two models represent a probabilistic \textit{encoder} and a probabilistic \textit{decoder}, respectively.

\begin{figure}[h]
	\centering
	\includegraphics[width=\linewidth]{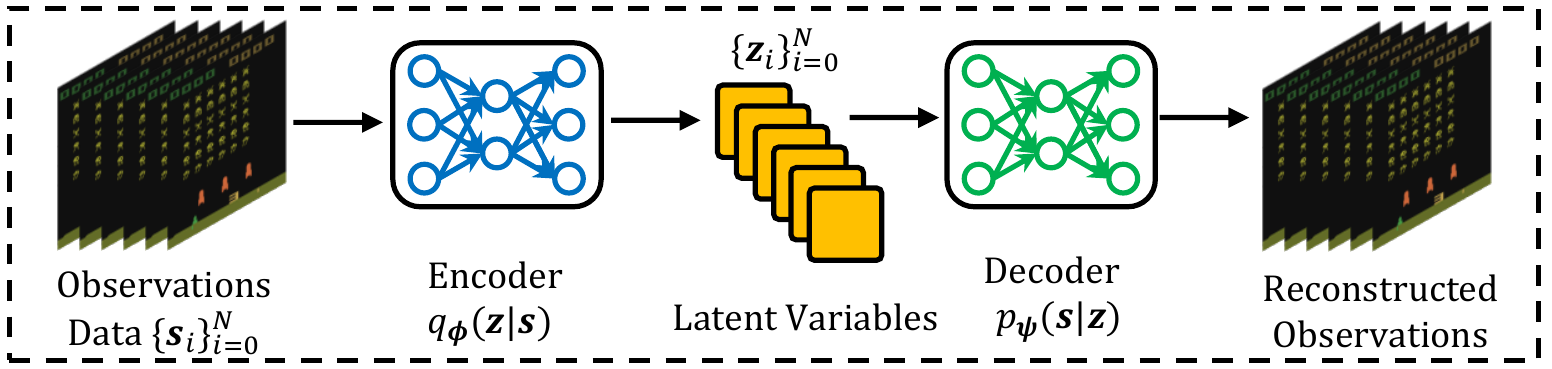}
	\caption{VAE model for embedding observations.}
	\label{fig:rise_a}
\end{figure}

We denote by $q_{\bm \phi}(\bm{z}|\bm{s})$ the recognition model represented by a neural network with parameters $\bm{\phi}$. The recognition model accepts an observation input before encoding the input into latent variables. Similarly, we represent the generative model as $p_{\bm \psi}(\bm{s}|\bm{z})$ using a neural network with parameters $\bm{\psi}$, accepting the latent variables and reconstructing the observation. Given a trajectory $\tau=\{\bm{s}_{0},\bm{a}_{0},\dots,\bm{a}_{T-1},\bm{s}_{T}\}$, the VAE model is trained by minimizing the following loss function:
\begin{equation}\label{eq:vae loss}
	\begin{aligned}
		L(\bm{s}_{t};{\bm \phi},{\bm \psi})&=\mathbb{E}_{q_{\bm \phi}(\bm{z}|\bm{s}_{t})}\big[\log p_{\bm \psi}(\bm{s}_{t}|\bm{z})\big]\\
		&-D_{{\rm KL}}\big(q_{\bm\phi}(\bm{z}|\bm{s}_{t})\Vert p_{\bm \psi}(\bm{z})\big),
	\end{aligned}
\end{equation}
where $t=0,\dots,T$, $D_{{\rm KL}}(\cdot)$ is the Kullback-Liebler (KL) divergence.

Next, we will elaborate on the design of the $k$ value to improve the estimation accuracy of the state entropy. \cite{singh2003nearest} investigated the performance of this entropy estimator for some specific probability distribution functions such as uniform distribution and Gaussian distribution. Their simulation results demonstrated that the estimation accuracy first increased before decreasing as the $k$ value increases. To circumvent this problem, we propose our $k$-value searching scheme as shown in Fig.~\ref{fig:rise_b}. We first divide the observation dataset into $K$ subsets before the encoder encodes the data into low-dimensional embedding vectors. Assuming that all the data samples are independent and identically distributed, an appropriate $k$ value should produce comparable results on different subsets. By exploiting this intuition, we propose to search the optimal $k$ value that minimizes the min-max ratio of entropy estimation set. Denote by $\pi_{\bm{\theta}}$ the policy network, the detailed searching algorithm is summarized in Algorithm~\ref{algo:k search}.

\begin{figure}[h]
	\centering
	\includegraphics[width=\linewidth]{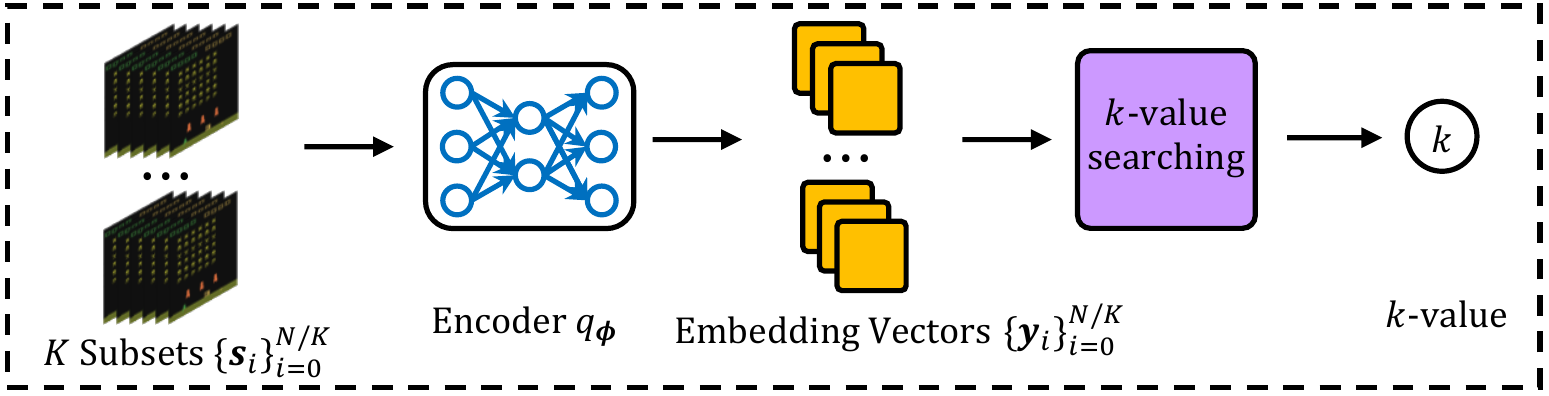}
	\caption{$k$-value searching.}
	\label{fig:rise_b}
\end{figure}

\begin{algorithm}[h]
	\caption{$k$-value searching method}
	\label{algo:k search}
	\begin{algorithmic}[1]
		\STATE Initialize a policy network $\pi_{\bm{\theta}}$;
		\STATE Initialize the number of sample steps $N$, the threshold $k_{\rm max}$ of $k$, a null array $\delta$ with length $k_{\rm max}$, and the number of subsets $K$;
		\STATE Execute policy $\pi_{\bm{\theta}}$ and collect the trajectory $\tau=\{\bm{s}_{0},\bm{a}_{0},\dots,\bm{a}_{N-1},\bm{s}_{N}\}$;
		\STATE Divide the observations dataset $\{\bm{s}_{i}\}_{i=0}^{N}$ into $K$ subsets randomly;
		\FOR {$k=1,2,\dots,k_{\rm max}$}
		\STATE Calculate the estimated entropy on $K$ subsets using Eq.~\eqref{eq:estimator}:
		\begin{equation}
			\hat{\bm{H}}_{k}=(\hat{H}_{k, N/K, \alpha}[1],\dots,\hat{H}_{k, N/K, \alpha}[K])
		\end{equation}
		
		\STATE Calculate the min-max ratio $\delta(\hat{\bm{H}}_{k})$ and $\bm{\delta}[k] \leftarrow \delta(\hat{\bm{H}}_{k})$;
		\ENDFOR
		\STATE Output $k=\underset{k}{\rm argmin}\:\bm{\delta}[k]$.
	\end{algorithmic}
\end{algorithm}

Finally, we are ready to propose our RISE framework by exploiting the optimal $k$ value derived above. As shown in Fig.~\ref{fig:rise_c}, the proposed RISE framework first encodes the high-dimensional observation data into low-dimensional embedding vectors through $q:\mathcal{S}\rightarrow\mathbb{R}^{m}$. After that, the Euclidean distance between $\bm{y}_{t}$ and its $k$-nearest neighbor is computed as the intrinsic reward. Algorithm~\ref{algo:rise on-policy} and Algorithm~\ref{algo:rise off-policy} summarize the on-policy and off-policy RL versions of the proposed RISE, respectively. In the off-policy version, the entropy estimation is performed on the sampled transitions in each step. As a result, a larger batch size can improve the estimation accuracy. It is worth pointing out that RISE can be straightforwardly integrated into any existing RL algorithms such as Q-learning and soft actor-critic, providing high-quality intrinsic rewards for improved exploration.

\begin{figure}[h]
	\centering
	\includegraphics[width=\linewidth]{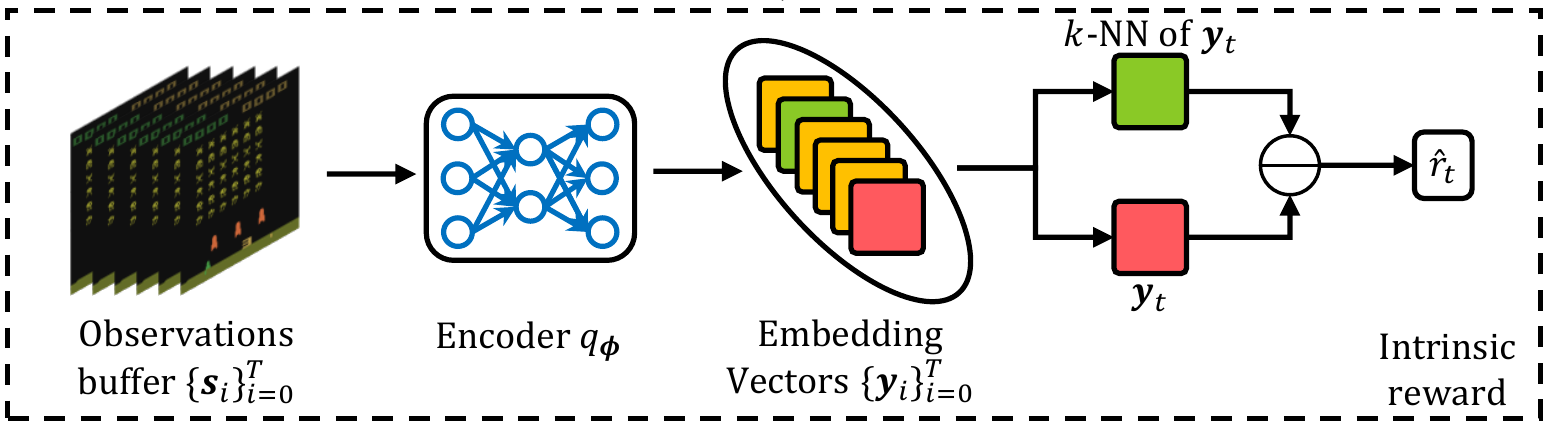}
	\caption{The generation of intrinsic rewards, where $k$-NN is $k$-nearest neighbor and $\ominus$ denotes the Euclidean distance.}
	\label{fig:rise_c}
\end{figure}

\begin{algorithm}[h]
	\caption{On-policy RL Version of RISE}
	\label{algo:rise on-policy}
	\begin{algorithmic}[1]
		\STATE \textbf{Phase 1: $k$-value searching and encoder training}
		\STATE Initialize the policy network $\pi_{\bm{\theta}}$, encoder $q_{\bm \phi}$ and decoder $p_{\bm \psi}$;
		\STATE Initialize the number of sample steps $N$, the threshold $k_{\rm max}$ of $k$-value, the embedding size $m$ and the number subsets $K$ ;
		\STATE Execute policy and collect observations data $\{\bm{s}_{i}\}_{i=1}^{N}$;
		\STATE Use $\{\bm{s}_{i}\}_{i=1}^{N}$ to train the encoder;
		\STATE Use Algorithm \ref{algo:k search} to select $k$-value;
		\STATE \textbf{Phase 2: Policy update}
		\STATE Initialize the maximum episodes $E$, order $\alpha$, coefficients $\lambda_{0},\zeta$ and decay rate $\kappa$;
		\FOR {episode $\ell=1,\dots,E$}
		\STATE Collect the trajectory $\tau_{\ell}=\{\bm{s}_{0},\bm{a}_{0},\dots,\bm{a}_{T-1},\bm{s}_{T}\}$;
		\STATE Compute the embedding vectors $\{\bm{z}_{t}\}_{t=0}^{T}$ of $\{\bm{s}_{t}\}_{t=0}^{T}$ using the encoder $q_{\bm \phi}$;
		\STATE Compute $\hat{r}(\bm{s}_{t})\leftarrow (\Vert \bm{z}_{t}-\tilde{\bm{z}}_{t} \Vert_{2})^{1-\alpha}$;
		\STATE Update $\lambda_{\ell}=\lambda_{0}(1-\kappa)^{\ell}$;
		\STATE Let $r^{\rm total}_{t}=r(\bm{s}_{t},\bm{a}_{t})+\lambda_{\ell} \cdot \hat{r}(\bm{s}_{t})+\zeta\cdot H(\pi(\cdot|\bm{s}_t))$;
		\STATE Update the policy network with transitions $\{\bm{s}_{t},\bm{a}_{t},\bm{s}_{t+1},r^{\rm total}_{t}\}_{t=0}^{T}$ using any on-policy RL algorithms.
		\ENDFOR
	\end{algorithmic}
\end{algorithm}

\begin{algorithm}[h]
	\caption{Off-policy RL Version of RISE}
	\label{algo:rise off-policy}
	\begin{algorithmic}[1]
		\STATE \textbf{Phase 1} of Algorithm \ref{algo:rise on-policy};
		\STATE \textbf{Phase 2: Policy update}
		\STATE Initialize the maximum environment steps $t_{\rm max}$, order $\alpha$, coefficients $\lambda_{0},\zeta$, decay rate $\kappa$, and replay buffer $\mathcal{B}\leftarrow \emptyset$;
		
		\FOR {step $t=1,\dots,t_{\rm max}$}
		\STATE Collect the transition $(\bm{s}_{t},\bm{a}_{t},r_{t},\bm{s}_{t+1})$ and let $\mathcal{B}\leftarrow \mathcal{B}\cup\{(\bm{s}_{t},\bm{a}_{t},r_{t},\bm{s}_{t+1},\bm{z}_{t})\}$, where $\bm{z}_{t}=q_{\bm \theta}(\bm{s}_{t})$;
		\STATE Sample a minibatch $\{(\bm{s}_{i},\bm{a}_{i},r_{i},\bm{s}_{i+1},\bm{z}_{i})\}_{i=1}^{B}$ from $\mathcal{B}$ randomly;
		\STATE Compute $\hat{r}(\bm{s}_{i})\leftarrow (\Vert \bm{z}_{i}-\tilde{\bm{z}}_{i} \Vert_{2})^{1-\alpha}$;
		\STATE Update $\lambda_{t}=\lambda_{0}(1-\kappa)^{t}$;
		\STATE Let $r^{\rm total}_{i}=r(\bm{s}_{i},\bm{a}_{i})+\lambda_{t} \cdot \hat{r}(\bm{s}_{i})+\zeta\cdot H(\pi(\cdot|\bm{s}_i))$;
		\STATE Update the policy network with transitions $\{\bm{s}_{i},\bm{a}_{i},\bm{s}_{i+1},r^{\rm total}_{i}\}_{i=1}^{B}$ using any off-policy RL algorithms.
		\ENDFOR
	\end{algorithmic}
\end{algorithm}

\subsection{Experiments}
In this section, we will evaluate our RISE framework on both the tabular setting and environments with high-dimensional observations. We compare RISE against two representative intrinsic reward-based methods, namely RE3 and Max\Ry \cite{zhang2021exploration}. We also train the agent without intrinsic rewards for ablation studies. As for hyper-parameters setting, we only report the values of the best experiment results.

\subsubsection{Maze Games}
In this section, we first leverage a simple but representative example to highlight the effectiveness of the \Ry state entropy-driven exploration. We introduce a grid-based environment Maze2D \cite{matthew2016github} illustrated in Fig.~\ref{fig:maze}. The agent can move one position at a time in one of the four directions, namely left, right, up, and down. The goal of the agent is to find the shortest path from the start point to the end point. In particular, the agent can teleport from a portal to another identical mark.

\begin{figure}[h]
	\centering
	\includegraphics[width=0.5\linewidth]{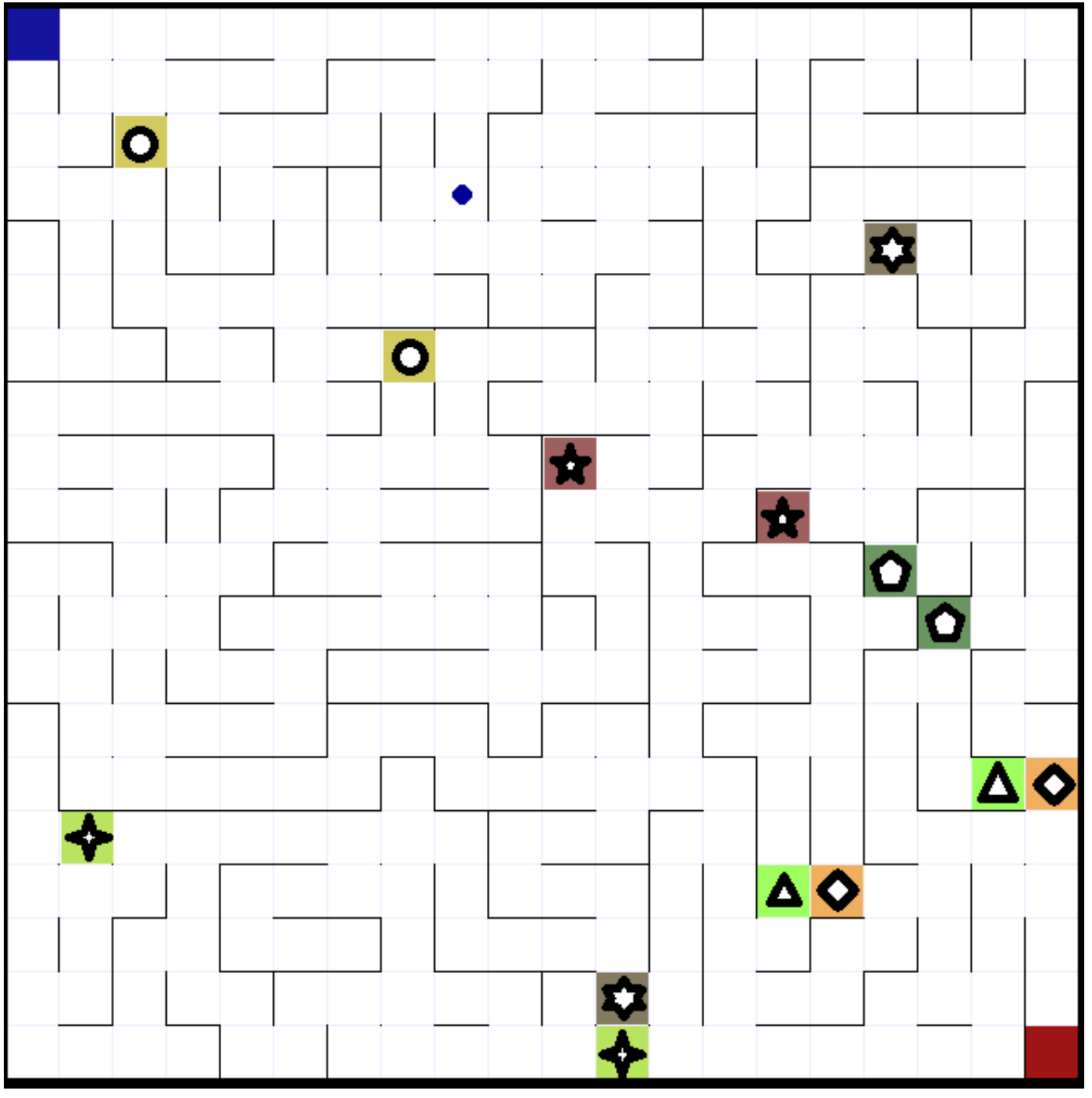}
	\caption{A maze game with grid size $20\times20$.}
	\label{fig:maze}
\end{figure}


\textbf{Experimental Setting}
The standard Q-learning (QL) algorithm \cite{watkins1992q} is selected as the benchmarking method. We perform extensive experiments on three mazes with different sizes. Note that the problem complexity increases exponentially with the maze size. In each episode, the maximum environment step size was set to $10M^{2}$, where $M$ is the maze size. We initialized the Q-table with zeros and updated the Q-table in every step for efficient training. The update formulation is given by:
\begin{equation}
	Q(\bm{s},\bm{a})\leftarrow Q(\bm{s},\bm{a})+\eta[r+\gamma \max_{\bm{a}'}Q(\bm{s}',\bm{a}')-Q(\bm{s},\bm{a})],
\end{equation}
where $Q(\bm{s},\bm{a})$ is the action-value function. The step size was set to $0.2$ while a $\epsilon$-greedy policy with an exploration rate of $0.001$ was employed.

\textbf{Performance Comparison}
\begin{figure}[h]
	\centering
	\includegraphics[width=0.55\linewidth]{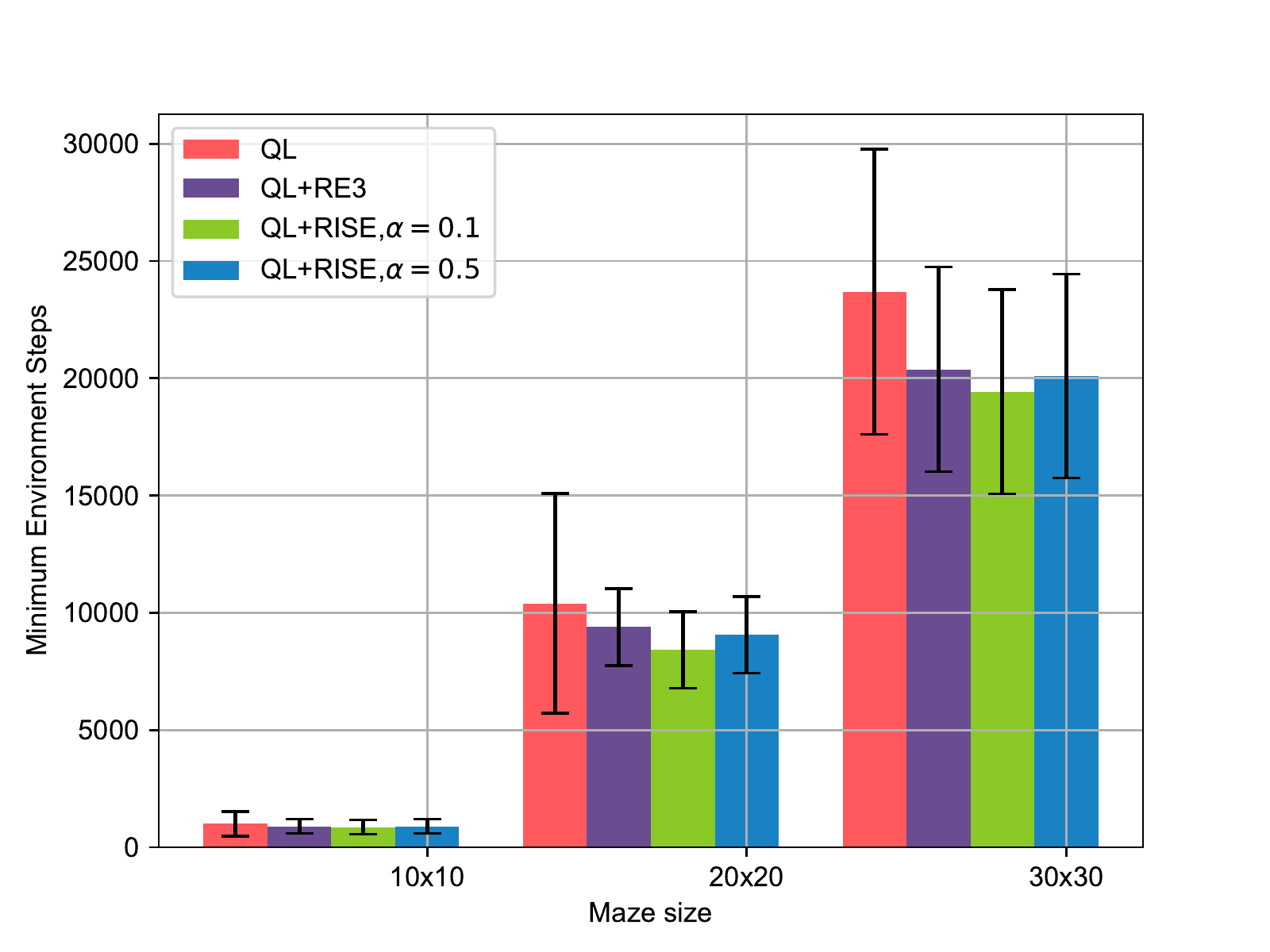}
	\caption{Average exploration performance comparison over $100$ simulation runs.}
	\label{fig:maze return}
\end{figure}

To compare the exploration performance, we choose the minimum number of environment steps taken to visit all states as the key performance indicator (KPI). For instance, a $10\times10$ maze of $100$ grids corresponds to $100$ states. The minimum number of steps for the agent to visit all the possible states is evaluated as its exploration performance. As seen in Fig.~\ref{fig:maze return}, the proposed \textit{Q-learning+RISE} achieved the best performance in all three maze games. Moreover, RISE with smaller $\alpha$ takes less steps to finish the exploration phase. This experiment confirmed the great capability of the \Ry state entropy-driven exploration.

\subsubsection{Atari Games}

Next, we will test RISE on the Atari games with a discrete action space, in which the player aims to achieve more points while remaining alive \cite{brockman2016openai}. To generate the observation of the agent, we stacked four consecutive frames as one input. These frames were cropped to the size of $(84, 84)$ to reduce the required computational complexity. 

\textbf{Experimental Setting}
To handle the graphic observations, we leveraged convolutional neural networks (CNNs) to build RISE and the benchmarking methods. For fair comparison, the same policy network and value network are employed for all the algorithms, and their architectures can be found in Table~\ref{tb:cnn na}. For instance, ``8$\times$8 Conv. 32" represents a convolutional layer that has $32$ filters of size 8$\times$8. A categorical distribution was used to sample an action based on the action probability of the stochastic policy. The VAE block of RISE and Max\Ry need to learn an encoder and a decoder. The encoder is composed of four convolutional layers and one dense layer, in which each convolutional layer is followed by a batch normalization (BN) layer \cite{ioffe2015batch}. Note that ``Dense 512 \& Dense 512" in Table~\ref{tb:cnn na} means that there are two branches to output the mean and variance of the latent variables, respectively. For the decoder, it utilizes four deconvolutional layers to perform upsampling while a dense layer and a convolutional layer are employed at the top and the bottom of the decoder. Finally, no BN layer is included in the decoder and the ReLU activation function is employed for all components.

\begin{figure*}[ht]
	\centering
	\includegraphics[width=\linewidth]{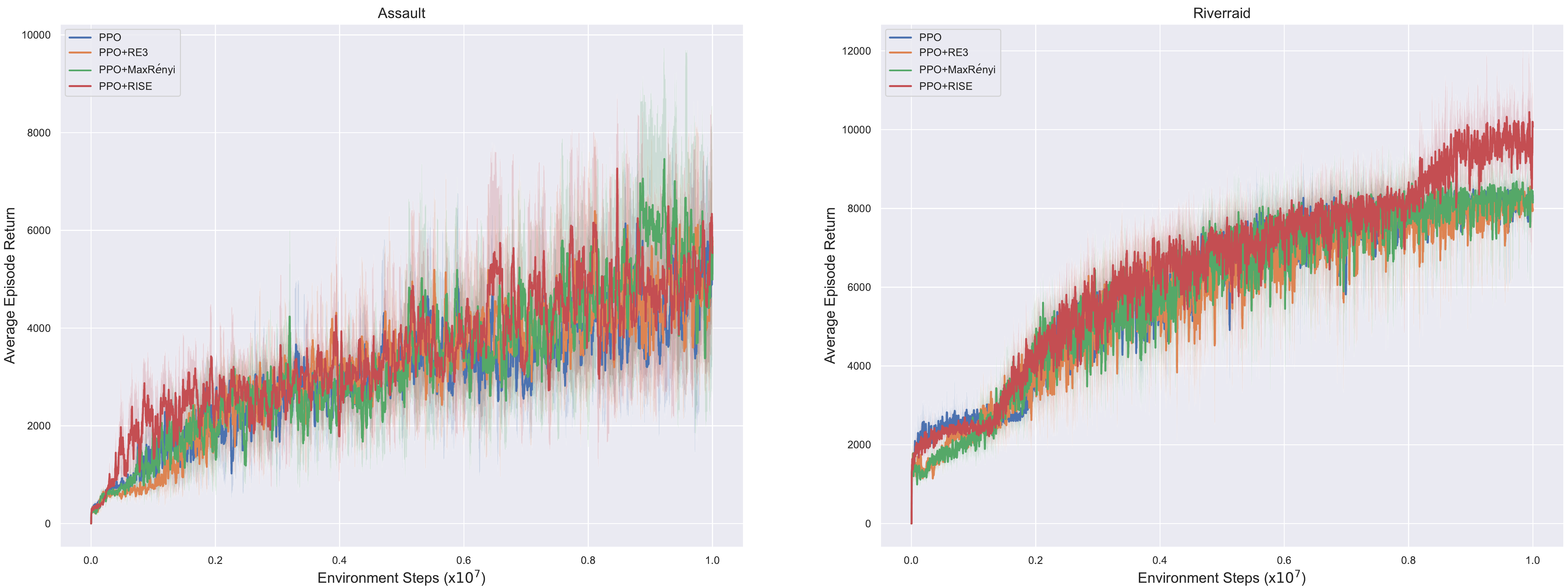}
	\caption{Average episode return versus number of environment steps on Atari games.}
	\label{fig:dis eps return}
\end{figure*}

\begin{table*}[h]
	\centering
	\caption{Performance comparison in nine Atari games.}
	\label{tb:dis final performance}
	\normalsize
	\begin{tabular}{l|l|l|l|l}
		\hline
		Game     & PPO     & PPO+RE3 & PPO+Max\Ry & PPO+RISE \\ \hline
		Assault  & 5.24k$\pm$1.86k   & 5.54k$\pm$2.12k   & 5.68k$\pm$1.99k   & \textbf{5.78k$\pm$2.44k}  \\
		Battle Zone & 18.63k$\pm$1.96k   & 20.47k$\pm$2.73k   & 19.17k$\pm$3.48k   & \textbf{21.80k$\pm$5.44k}  \\
		Demon Attack  & 13.55k$\pm$4.65k   & 17.40k$\pm$9.63k   & 16.18k$\pm$8.19k   & \textbf{18.39k$\pm$7.61k} \\
		Kung Fu Master & 21.86k$\pm$8.92k   & 23.21k$\pm$5.74k   & 27.20k$\pm$5.59k   & \textbf{27.76k$\pm$5.86k} \\
		Riverraid    & 7.99k$\pm$0.53k   & 8.14k$\pm$0.37k   & 8.21k$\pm$0.36k   & \textbf{10.07k$\pm$0.77k} \\
		Space Invaders & 0.72k$\pm$0.25k   & 0.89k$\pm$0.33k   & 0.81k$\pm$0.37k   & \textbf{1.09k$\pm$0.21k}\\
		\hline
	\end{tabular}
\end{table*}

In the first phase, we initialized a policy network $\pi_{\bm \theta}$ and let it interact with eight parallel environments with different random seeds. We first collected observation data over ten thousand environment steps before the VAE encoder generates the latent vectors of dimension of $128$ from the observation data. After that, the latent vectors were sent to the decoder to reconstruct the observation tensors. For parameters update, we used an Adam optimizer with a learning rate of $0.005$ and a batch size of $256$. Finally, we divided the observation dataset into $K=8$ subsets before searching for the optimal $k$-value over the range of $[1,15]$ using Algorithm~\ref{algo:k search}.

Equipped with the learned $k$ and encoder $q_{\bm \phi}$, we trained RISE with ten million environment steps. In each episode, the agent was also set to interact with eight parallel environments with different random seeds. Each episode has a length of $128$ steps, producing $1024$ pieces of transitions. After that, we calculated the intrinsic reward for all transitions using Eq.~\eqref{eq:intrinsic reward}, where $\alpha=0.1, \lambda_{0}=0.1$. Finally, the policy network was updated using a proximal policy optimization (PPO) method \cite{schulman2017proximal}. More specifically, we used a PyTorch implementation of the PPO method, which can be found in \cite{kostrikov2018github}. The PPO method was trained with a learning rate of $0.0025$, a value function coefficient of $0.5$, an action entropy coefficient of $0.01$, and a generalized-advantage-estimation (GAE) parameter of $0.95$ \cite{schulman2015high}. In particular, a gradient clipping operation with threshold $[-5, 5]$ was performed to stabilize the learning procedure. As for benchmarking methods, we trained them following their default settings reported in the literature \cite{zhang2021exploration,seo2021state}.

\textbf{Performance Comparison}

The average one-life return is employed as the KPI in our performance comparison. Table~\ref{tb:dis final performance} illustrates the performance comparison over eight random seeds on nine Atari games. For instance, 5.24k$\pm$1.86k represents the mean return is $5.24$k and the standard deviation is $1.86$k. The highest performance is shown in bold. As shown in Table~\ref{tb:dis final performance}, RISE achieved the highest performance in all nine games. Both RE3 and Max\Ry achieved the second highest performance in three games. Furthermore, Fig.~\ref{fig:dis eps return} illustrates the change of average episode return during training of two selected games. It is clear that the growth rate of RISE is faster than all the benchmarking methods.

Next, we compare the training efficiency between RISE and the benchmarking methods, and the frame per second (FPS) is set as the KPI. For instance, if a method takes $10$ second to finish the training of an episode, the FPS is computed as the ratio between the time cost and episode length. The time cost involves only interaction and policy updates for the vanilla PPO agent. But the time cost needs to involve further the intrinsic reward generation and auxiliary model updates for other methods. As shown in Fig.~\ref{fig:fps}, the vanilla PPO method achieves the highest computation efficiency, while RISE and RE3 achieve the second highest FPS. In contrast, Max\Ry has far less FPS that RISE and RE3. This mainly because RISE and RE3 require no auxiliary models, while Max\Ry uses a VAE to estimate the probability density function. Therefore, RISE has great advantages in both the policy performance and learning efficiency.

\begin{figure}[h]
	\centering
	\includegraphics[width=0.6\linewidth]{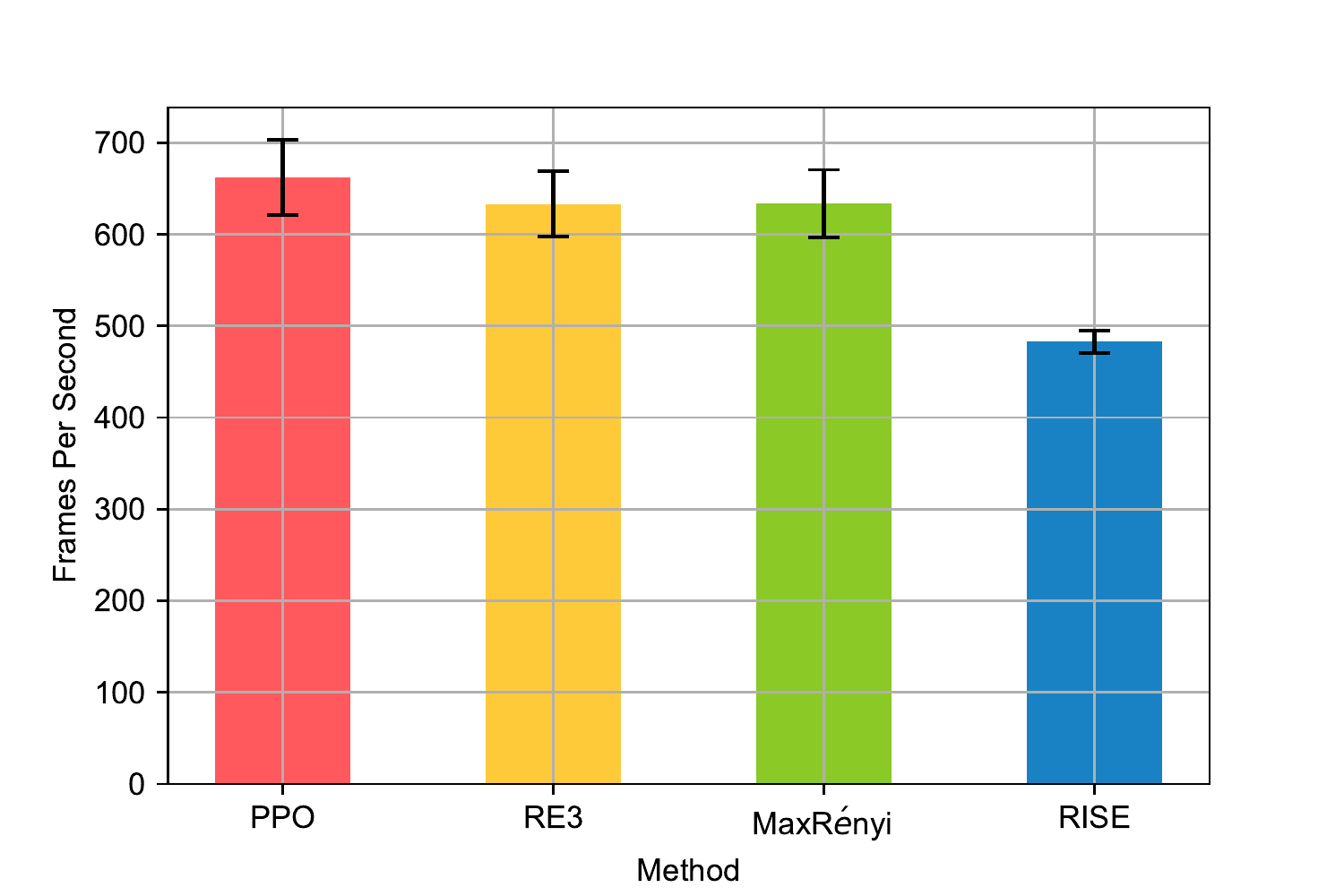}
	\caption{Average computational complexity on Atari games. The experiments were performed in Ubuntu 18.04 LTS operating system with a Intel 10900x CPU and a NVIDIA RTX3090 GPU.}
	\label{fig:fps}
\end{figure}

\subsubsection{Bullet Games}
\textbf{Experimental Setting}
Finally, we tested RISE on six Bullet games \cite{coumans2016pybullet} with continuous action space, namely \textit{Ant}, \textit{Half Cheetah}, \textit{Hopper}, \textit{Humanoid}, \textit{Inverted Pendulum} and \textit{Walker 2D}. In all six games, the target of the agent is to move forward as fast as possible without falling to the ground. Unlike the Atari games that have graphic observations, the Bullet games use fixed-length vectors as observations. For instance, the ``Ant" game uses $28$ parameters to describe the state of the agent, and its action is a vector of $8$ values within $[-1.0, 1.0]$.

\begin{table*}[ht]
	\centering
	\caption{Performance comparison of six Bullet games.}
	\label{tb:con final performance}
	\normalsize
	\begin{tabular}{l|l|l|l|l}
		\hline
		Game     & PPO     & PPO+RE3 & PPO+Max\Ry & PPO+RISE \\ \hline
		Ant               & 2.25k$\pm$0.06k  & 2.36k$\pm$0.01k  & 2.43k$\pm$0.03k  & \textbf{2.71k$\pm$0.07k}  \\
		Half Cheetah      & 2.36k$\pm$0.02k &  2.41k$\pm$0.02k  & 2.40k$\pm$0.01k  & \textbf{2.47k$\pm$0.07k}  \\
		Hopper   		  & 1.53k$\pm$0.57k  & 2.08k$\pm$0.55k  & 2.23k$\pm$0.31k  &  \textbf{2.44k$\pm$0.04k}  \\
		Humanoid          &  0.83k$\pm$0.25k & 0.95k$\pm$0.64k  & 1.17k$\pm$0.54k &  \textbf{1.24k$\pm$0.92k}  \\
		Inverted Pendulum & 1.00k$\pm$0.00k  & 1.00k$\pm$0.00k  & 1.00k$\pm$0.00k  & \textbf{1.00k$\pm$0.00k} \\
		Walker 2D         & 1.66k$\pm$0.37k  & 1.85k$\pm$0.71k  & 1.73k$\pm$0.21k  & \textbf{1.96k$\pm$0.34k} \\
		\hline
	\end{tabular}
\end{table*}

We leveraged the multilayer perceptron (MLP) to implement RISE and the benchmarking methods. The detailed network architectures are illustrated in Table~\ref{tb:mlp na}. Note that the encoder and decoder were designed for MaxR\'enyi, and no BN layers were introduced in this experiment. Since the state space is far simpler than the Atari games, the entropy can be directly derived from the observations while the training procedure for the encoder is omitted. We trained RISE with ten million environment steps. The agent was also set to interact with eight parallel environments with different random seeds, and Gaussian distribution was used to sample actions. The rest of the updating procedure was consistent with the experiments of the Atari games.

\textbf{Performance Comparison}
Table~\ref{tb:con final performance} illustrates the performance comparison between RISE and the benchmarking methods. Inspection of Table~\ref{tb:con final performance} suggests that RISE achieved the best performance in all six games. In summary, RISE has shown great potential for achieving excellent performance in both discrete and continuous control tasks.

\section{Conclusion}
\label{chap:con}

\subsection{Contributions}

In this thesis, we have investigated the problem of improving exploration in RL. We first systematically analyze the exploration-exploitation dilemma via multi-armed bandit problem before introducing several classical exploration strategies, such as Thompson sampling and noise-based exploration. Since those methods cannot effectively adapt to hard-exploration environments, we further introduce the intrinsically-motivated RL that utilizes intrinsic learning motivation to encourage exploration. We carefully classify the existing intrinsic reward methods, and analyzed their practical drawbacks, such as vanishing rewards and noisy-TV problem. Finally, we proposed a brand new intrinsic reward method via R\'enyi state entropy maximization that overcomes the drawbacks of the preceding methods and provides powerful exploration incentives. Extensive simulation is performed to compare the performance of RISE against existing methods using both discrete and continuous control tasks as well as several hard exploration games. Simulation results confirm that the proposed module achieve superior performance with high efficiency and robustness.

\subsection{Future Work}

Intrinsically-motivated RL has achieved great progresses in hard-exploration problem. However, there remains several critical challenges in its consequent development. Most existing intrinsic reward methods cannot guarantee the policy invariance. Despite they improve the exploration performance, the agent may never learn the desired optimal policy. On other hand, the mechanism of intrinsic reward still lacks rigorous mathematical interpreability, which greatly limits its practical application and development. In particular, the existing RL algorithms still produce ultra-high computation complexity, even with the promotion of those exploration methods. Therefore, more research efforts should be devoted to improving the efficiency of these algorithms. This paper is expected to inspire more subsequent research.

\clearpage

\appendix
\section{Experimental Settings} 

\subsection{Network Architectures}
\begin{table}[h]
	\caption{The CNN-based network architectures.}
	\label{tb:cnn na}
	\centering
	\small
	\begin{tabular}{l|l|l}
		\hline
		\textbf{Module} & Policy network $\pi_{\bm \theta}$                                                                                                                                                                                                 & Value network                                                                                                                                                                    \\ \hline
		Input  & States                                                                                                                                                                                                & States                                                                                                                                                                           \\ \hline
		Arch.  & \begin{tabular}[c]{@{}l@{}}8$\times$8 Conv 32, ReLU\\ 4$\times$4 Conv 64, ReLU\\ 3$\times$3 Conv 32, ReLU\\ Flatten\\ Dense 512, ReLU\\ Dense $|\mathcal{A}|$\\ Categorical Distribution\end{tabular} & \begin{tabular}[c]{@{}l@{}}8$\times$8 Conv 32, ReLU\\ 4$\times$4 Conv 64, ReLU\\ 3$\times$3 Conv 32, ReLU\\ Flatten\\ Dense 512, ReLU\\ Dense 1\end{tabular}                     \\ \hline
		Output & Actions                                                                                                                                                                                               & Predicted values                                                                                                                                                                           \\ \hline
		\textbf{Module} & Encoder $p_{\bm \psi}$                                                                                                                                                                                                       & Decoder $q_{\bm \phi}$                                                                                                                                                    \\ \hline
		Input  & States                                                                                                                                                                  & Latent variables                                                                                                                                                                           \\ \hline
		Arch.  & \begin{tabular}[c]{@{}l@{}}3$\times$3 Conv. 32, ReLU\\ 3$\times$3 Conv. 32, ReLU\\ 3$\times$3 Conv. 32, ReLU\\ 3$\times$3 Conv. 32\\ Flatten\\ Dense 512 \& Dense 512\\ Gaussian sampling\end{tabular} & \begin{tabular}[c]{@{}l@{}}Dense 64, ReLU\\ Dense 1024, ReLU\\ Reshape\\ 3$\times$3 Deconv. 64, ReLU\\ 3$\times$3 Deconv. 64, ReLU\\ 3$\times$3 Deconv. 64, ReLU\\ 8$\times$8 Deconv. 32\\ 1$\times$1 Conv. 4\end{tabular} \\ \hline
		Output & Latent variables                                                                                                                                                                                                & Predicted states                                                                                                                                                                \\ \hline
	\end{tabular}
\end{table}

\begin{table}[h]
	\caption{The MLP-based network architectures.}
	\label{tb:mlp na}
	\centering
	\small
	\begin{tabular}{l|l|l}
		\hline
		\textbf{Module} & Policy network $\pi_{\bm \theta}$                                                                                                                                                                                                 & Value network                                                                                                                                                                    \\ \hline
		Input  & States                                                                                                                                                                                                & States                                                                                                                                                                           \\ \hline
		Arch.  & \begin{tabular}[c]{@{}l@{}}Dense 64, Tanh\\ Dense 64, Tanh\\ Dense $|\mathcal{A}|$\\ Gaussian Distribution\end{tabular} & \begin{tabular}[c]{@{}l@{}}Dense 64, Tanh\\ Dense 64, Tanh\\ Dense 1 \end{tabular}                     \\ \hline
		Output & Actions                                                                                                                                                                                               & Predicted values                                                                                                                                                                           \\ \hline
		\textbf{Module} & Encoder $p_{\bm \psi}$                                                                                                                                                                                                       & Decoder $q_{\bm \phi}$                                                                                                                                                    \\ \hline
		Input  & States                                                                                                                                                                  & Latent variables                                                                                                                                                                           \\ \hline
		Arch.  & \begin{tabular}[c]{@{}l@{}}Dense 32, Tanh\\ Dense 64, Tanh\\ Dense 256\\ Dense 256 \& Dense 512\\ Gaussian sampling\end{tabular} & \begin{tabular}[c]{@{}l@{}}Dense 32, Tanh\\ Dense 64, Tanh\\ Dense observation shape\end{tabular} \\ \hline
		Output & Latent variables                                                                                                                                                                                                & Predicted states                                                                                                                                                                \\ \hline
	\end{tabular}
\end{table}

\subsection{Hyper-parameter Settings}

\begin{table}[h]
	\centering
	\caption{Hyperparameters of RISE used for Atari games.}
	\label{tb:hyperparameters RL atari}
	\begin{tabular}{l|ll}
		\hline
		\textbf{Phase}   & \textbf{Hyperparameter}   & \textbf{Value}      \\ \hline
		\multirow{9}{*}{\begin{tabular}[c]{@{}l@{}}$k$-value searching \\ and encoder training\end{tabular}}  & Number of parallel environments    & 8          \\
		& Number of sample steps $N$         & 1e+5       \\
		& Threshold of $k$-value $k_{\rm max}$            & 15         \\
		& Embedding size $m$                 & 128        \\
		& Number of subsets $K$               & 8          \\
		& Learning rate                      & 5e-3       \\
		& Batch size                         & 256        \\
		& Coefficient of divergence $\lambda_{D}$         & 1.0        \\
		& Coefficient of reconstruction loss $\lambda_{R}$ & 1.0        \\ \hline
		\multirow{11}{*}{Policy update} & Maximum environment steps          & 1e+7       \\
		& Observation size                   & (84, 84)   \\
		& Stacked frames                     & 4          \\
		& Frame skipping                     & False      \\
		& Learning rate for actor \& critic  & 2.5e-4      \\
		& Batch size                         & 256        \\
		& $k$                                & 5          \\
		& Order of \Ry entropy $\alpha$                     & 0.1        \\
		& Initial coefficient of intrinsic reward $\beta_{0}$     & 0.1     \\
		& Decay rate $\kappa$                & 1e-5       \\
		& Coefficient of GAE                 & 0.95       \\
		& Coefficient of value function      & 0.05       \\
		& Gradient clipping threshold        & {[}-5,5{]} \\ \hline
	\end{tabular}
\end{table}

\begin{table}[t]
	\centering
	\caption{Hyperparameters of RISE used for Bullet games.}
	\label{tb:hyperparameters RL bullet}
	\begin{tabular}{l|ll}
		\hline
		\textbf{Phase}   & \textbf{Hyperparameter}   & \textbf{Value}      \\ \hline
		\multirow{4}{*}{\begin{tabular}[c]{@{}l@{}}$k$-value searching \\ and encoder training\end{tabular}}  & Number of parallel environments    & 8          \\
		& Number of sample steps             & 1e+5       \\
		& Threshold of $k$-value             & 15         \\
		& Number of subsets $K$               & 8          \\ \hline
		\multirow{9}{*}{Policy update} & Maximum environment steps          & 1e+7       \\
		& Learning rate for actor \& critic  & 2.5-4      \\
		& Batch size                         & 512        \\
		& $k$                                & 5          \\
		& Order of \Ry entropy $\alpha$                     & 0.1        \\
		& Coefficient of intrinsic reward $\beta_{0}$     & 0.1     \\
		& Decay rate $\kappa$                & 1e-5       \\
		& Coefficient of GAE                 & 0.95       \\
		& Coefficient of value function      & 0.05       \\
		& Gradient clipping threshold        & [-5,5] \\ \hline
	\end{tabular}
\end{table}


\end{document}